\documentclass{article}

\usepackage{fullpage}

\usepackage{amssymb, amsmath}
\usepackage{tikz}
\usepackage[vlined,ruled,linesnumbered]{algorithm2e}
\usepackage{pgfplots}
\usepackage{graphicx}
\usepackage{dsfont}
\usepackage{tikz}
\usepackage{algorithmic}
\newcommand{\B}[1]{\mathbf{#1}} 
\usepackage{lineno,hyperref}
\RequirePackage[numbers]{natbib}
\usepackage{footmisc}
\usepackage{tikz}
\usepackage{ulem}

\newtheorem{assumption}{Assumption}

\usepackage{changepage}
\usepackage{graphicx}

\usepackage{mathrsfs} 
\usepackage{booktabs} 
\usepackage{tabularx} 
\usepackage{float} 
\usepackage{multirow} 
\usepackage{natbib} 
\usepackage{cancel} 
\usepackage{balance}
\usepackage{hyperref}
\usepackage{lineno}
\usepackage{hyperref}

\usepackage{amsfonts} 
\usepackage{amssymb} 
\usepackage{bm} 
\usepackage{dsfont} 

\usepackage{xcolor}

\newtheorem{lemma}{Lemma}

\newtheorem{proposition}{Proposition}
\newtheorem{theorem}{Theorem}
\newtheorem{proof}{Proof}

\newtheorem{remark}{Remark}

\author{Sami Alkhoury, Marianne Clausel, Emilie Devijver, Eric Gaussier, Alexandre Seiller
}

\title{Ensembles of Probabilistic Regression Trees}

\begin{document}

\maketitle

\begin{abstract}
Tree-based ensemble methods such as random forests, gradient-boosted trees, and Bayesian additive regression trees have been successfully used for regression problems in many applications and research studies. In this paper, we study ensemble versions of probabilistic regression trees that provide smooth approximations of the objective function by assigning each observation to each region with respect to a probability distribution. 
We prove that the ensemble versions of probabilistic regression trees considered are consistent, and experimentally study their bias-variance trade-off and compare them with the state-of-the-art in terms of performance prediction.
 \end{abstract}

\section{Introduction}
Regression trees \citep{breiman1984cart} and the ensemble methods based on them such as random forests \citep{breiman2001random}, gradient-boosted trees \citep{friedman2000,elith2008} and Bayesian additive regression trees \citep{bart_2019} have been successfully used for regression problems in many applications. For regression trees, the feature space is partitioned into a set of hyper-rectangles, and a constant model is fitted in each region. Thus, the prediction function is piecewise constant with the hard assignment of data points to the regions. As a result, standard regression trees may have difficulty adapting to the smoothness of the link functions and the noise in the input data.

Extensions of regression trees have been proposed to generalize this prediction function. 
Soft trees \citep{irsoy2012soft} and fuzzy trees \citep{suarez1999} are both used for classification and regression and can learn a parameter vector at each node, the dimensionality of which is equal to that of the input data. For a specific node, this vector is used in a gating function: it gives the probability for each observation to be assigned to the left children of the node. Each example is thus assigned to all leaves with a certain class membership, and the final prediction is a smooth combination of the prediction at each node. Soft and fuzzy trees can be seen as a direct extension of the hierarchical mixtures of experts (HME, \citet{Jordan:1994:HME:188104.188106}): indeed, if the HME use predefined trees or trees learned from another method (typically, a standard decision regression tree), then soft trees are constructed based on the hierarchy of experts. 
Smooth transition regression (STR) trees, introduced in \cite{rosa2008}, follow the same general principle but instead rely on a single parameter at each node. A sigmoid-based gating function is also used to assign points to different regions of the tree. 
Probabilistic regression (PR) trees \citep{Alkhoury_2020} are based on a simpler model (with respect to the number of parameters), where smoothness is given by the soft assignment of each datapoint to each region. 

Instead of focusing on one tree, ensemble methods have been proposed to improve regression and classification tasks. The most well-known ensemble methods based on regression trees are certainly random forests (RF) introduced in \cite{breiman2001random}, where small trees are averaged to reduce the variance, and gradient-boosted trees (GBT) \citep{friedman2000,elith2008} with an additive method, where each new tree reduces the resulting error, thus reducing the bias. Soft trees, STR trees, and PR trees, viewed as construction blocks, can also be used in ensemble extensions to reduce the bias or variance and thus improve the global performance. 
More recently, the ensemble method of Bayesian additive regression trees known as BART \citep{bart} has been proposed. As an ensemble method, many trees are combined. The Bayesian \textit{a priori} is used to define the structure of each tree and the parameters necessary to define each one. The boosting model is used to reduce the error, although an overall average is also calculated to reduce the variance. 
The extension to soft trees has been proposed, namely soft-BART \citep{linero2018bayesian}, which also allows for sparsity using a Dirichlet \textit{a priori }on the feature space.

From a theoretical viewpoint, consistency has been achieved for regression trees \citep{Gyorfi-2002}, standard RF \citep{scornet:2015}, boosting extensions \citep{ZhangYu2005}, and Bayesian extension of standard regression trees \cite{rockova2020}. It has also been derived in the specific case of the Bayesian extension of soft trees \citep{Linero_2018}.

In this paper, we study ensemble versions of PR trees through bagging, boosting and Bayesian additive regression. Although PR-RF, the random forest extension of PR trees, and PR-GBT, the gradient boosted version of PR trees, were proposed in the original PR tree paper \citep{Alkhoury_2020}, there were not studied in detail, both in terms of consistency and performance. Using results established by \citet{scornet:2015} for random forests and \citet{ZhangYu2005} for gradient boosting, it is relatively easy to show that both PR-RF and PR-GBT are consistent in a sense that will be made explicit later (see Section~\ref{sec:bag-boost}). We furthermore illustrate their practical applicability in terms of bias and variance with respect to standard regression trees. Establishing the consistency of probabilistic Bayesian additive regression trees (P-BART) is however more demanding (see Section~\ref{sec:BART}). Experiments performed in this latter context furthermore require a data set of moderate size compared with \cite{Alkhoury_2020}, as the Bayesian routine is more time-consuming.

The remainder of the paper is organized as follows: Section~\ref{sec:PRtree} introduces PR trees. Then, Section \ref{sec:bag-boost} introduces the bagging and boosting extensions of PR trees, namely PR-RF and PR-GBT, as well as theoretical guarantees of their consistency. Section \ref{sec:BART} deals with the Bayesian extension named P-BART. Section~\ref{sec:exps} presents the experiments conducted to illustrate the good performance of the prediction with respect to the state-of-the-art and show the bias-variance trade-off for each estimator. An archive containing the complete code is available in Git\footnote{\url{https://gitlab.com/sami.kh/pr-tree}}. Finally, Section~\ref{sec:concl} concludes the paper.

In what follows, a '.' in one dimension of a matrix means that we consider the corresponding row or column vector.

\section{Probabilistic Regression ((PR) trees}\label{sec:PRtree}
\subsection{Model}
\label{sec:model}

Let $\B{X} = (X_1,\cdots,X_p)$ be a $p$-dimensional input random vector. We assume the following:
\begin{assumption}
\label{ass:Xcompact}
$\B{X} = (X_1,\cdots,X_p)$ is almost surely in $[0,1]^p$. 
\end{assumption}
This holds true for an assumption such as compact subspace $\mathcal{X}$ of $\mathbb{R}^p$ instead of $[0,1]^p$, but for the sake of simplicity, we keep $[0,1]^p$ throughout the paper. 
Let $Y$ be an output random variable linked to $\B{X}$ through:
\begin{equation}\label{eq:reg}
Y = f(\B{X};\Theta) + \varepsilon_Y, \, {\varepsilon}_Y \sim \mathcal{N}(0, \tilde \sigma^2),
\end{equation}
where $\Theta$ is the set of parameters on which $f$ relies.

For a partition of $\mathbb{R}^p$ into $K$ hyper-rectangles, referred to as regions and denoted as $\mathcal{R}_k=[a_{k,1},b_{k,1}] \times ... \times [a_{k,p},b_{k,p}]_{1\leq k \leq K}$, obtained by dyadic splits, the standard regression trees focus on the predictor with the form, for $\B{x}\in \mathbb{R}^p$, $f(\B{x};\Theta) = \sum_{k=1}^K \gamma_k \mathds{1}_{\lbrace \B{x} \in \mathcal{R}_k \rbrace}$ for weights $\gamma_k$ associated with the $k$-th region $\mathcal{R}_k$, where $\Theta = ((\mathcal{R}_k)_{1\leq k \leq K},\boldsymbol{\gamma})$. For $1\le k\le K$, $(a_{k,j}, b_{k,j}) \in (-\infty, +\infty)^2$ and the segments can be extended when $a_{k,j} = \pm \infty$ or $b_{k,j} = \pm \infty$, $1 \le j \le p$.
To generalize this, PR trees replace the indicator function with a function $\Psi$: for $\B{x}\in \mathbb{R}^p$, 
\begin{align}\label{treeProbabilistic}
f_{\text{PR}} \left(\B{x}; \Theta\right) = \sum_{k=1}^K \gamma_k \Psi(\B{x};\mathcal{R}_{k},\boldsymbol{\sigma}).
\end{align}
The set of parameters to be estimated is thus $\Theta = ((\mathcal{R}_k)_{1\leq k \leq K},\boldsymbol{\gamma},\boldsymbol{\sigma})$, where $(\mathcal{R}_k)_{1\leq k \leq K}$ corresponds to the set of regions, the associated weights are represented by $\boldsymbol{\gamma} = (\gamma_k)_{1\leq k \leq K} \in \mathbb{R}^K$, and the noise in the input variables is captured in $\boldsymbol{\sigma} \in \mathbb{R}_+^p$.
When $\Psi(\B{x};\mathcal{R}_k,\boldsymbol{\sigma}) = \mathds{1}_{\lbrace \B{x} \in \mathcal{R}_k \rbrace}, \, \forall k$, $1\leq k \leq K$, standard regression trees are obtained.

Using a probability density function $\phi$, the functions $\Psi$ considered here link the data points to different regions of the tree and smooth the predictions made. For all $\B{x} \in [0,1]^p$, they are defined by:
\begin{align}\label{Psi}
\Psi(\B{x};\mathcal{R}_{k},\boldsymbol{\sigma}) = \frac{1}{\prod_{j=1}^p\sigma_j} \int_{\mathcal{R}_k} \phi\left(\left(\frac{u_j-x_j}{\sigma_j}\right)_{1\le j \le p}\right)d\mathbf u.
\end{align}
This formula can be interpreted as introducing a latent variable $\mathbf{U}$ such that:
$$ \mathbf{X}= \mathbf{U} + \Sigma \varepsilon_{\mathbf{X}}$$
with $\Sigma = \text{diag}(\sigma_1,\ldots, \sigma_p)$ and $(\varepsilon_{\mathbf{X}})_{j}$ for $1\leq j \leq p$ iid from the distribution $\phi$.
 $\mathbf{X}$ can be considered to be the measured values of the input variables, possibly with additional noise, whereas $\mathbf{U}$ is the true value (without noise) of the input variables.
In practice, experts may have empirical knowledge about the nature of the errors (for example, when measurements are performed by calibrated machines), which can help to choose $\phi$. 

As an example, considering for $\phi$ the multivariate Gaussian distribution with a diagonal covariance matrix defined by $\boldsymbol{\sigma}$ , this leads to:
\begin{align}
\Psi(\B{x};\mathcal{R}_k,\boldsymbol{\sigma}) = \prod_{j=1}^p \frac{1}{\sigma_j \sqrt{2\pi}} \int_{a_{k,j}}^{b_{k,j}} e^{-\frac{(u-x_j)^2}{2\sigma_j^2}} du.
\label{alloc-function}
\end{align}
In this case, defining the distance between an observation $\B{x}$ and a region $\mathcal{R}_k$ by $d(\B{x},\mathcal{R}_k) = \displaystyle \inf_{\B{z} \in \mathcal{R}_k} ||\B{x} - \B{z}||_2^2$, the closer $\B{x}$ is to region $\mathcal{R}_k$, the greater the contribution of $\mathcal{R}_k$ is to the prediction of $Y$ given $\B{x}$.
Note that if $\sigma_j \rightarrow 0$, for all $1\leq j \leq p$, PR trees (or formula (4) but less smooth) come back to standard regression trees.

\subsection{Estimation and algorithm}
Given a training set $ {\cal D}_n=\left\{(\B{x}^{(i)}, y^{(i)})_{1\leq i \leq n}\right\}$, with $\B{x} \in \mathbb{R}^p, \, y \in \mathbb{R}$, and in accordance with the empirical risk minimization principle with a quadratic loss, the estimation procedure for probabilistic regression trees followed here aims to find the parameters $\hat\Theta_n$ as solutions for:
\begin{align}\label{emp-risk}
\underset{{\Theta}}{\operatorname{argmin}} \sum_{i=1}^n \left(y^{(i)} - \sum_{k=1}^K \gamma_k P_{ik}\right)^2,
\end{align}
with $P_{ik}:=\Psi(\B{x}^{(i)};\mathcal{R}_{k},\boldsymbol{\sigma})$. The $n\times K$ matrix $\boldsymbol{P}$ thus encodes the relations between each training example $\B{x}^{(i)}$ and each region $\mathcal{R}_k$ such that $0 \le P_{ik} \le 1$ and $\forall i, \, 1 \le i \le n, \sum_{k=1}^K P_{ik}=1$.

As in standard regression trees, the regions and weights are estimated in the following way: for a fixed $\boldsymbol{\sigma}$, one alternates between region and weight estimates until reaching a stopping criterion (any standard stopping criterion can be used here such as tree depth or number of examples in a leaf). During this process, the number of regions increases, and the matrix $\boldsymbol{P}$ and weights $\boldsymbol{\gamma}$ are gradually updated.

When fixing the regions $(\mathcal{R}_k)_{1\leq k \leq K}$ and the vector $\boldsymbol{\sigma}$, minimizing Eq.~\eqref{emp-risk} with respect to $\boldsymbol{\gamma}$ leads to a linear regression between $(y^{(i)})_{1\leq i \leq n}$ and $\mathbf P$. Note that this is once again the case for the standard regression tree, where the local mean can also be seen as a regression problem.

To estimate $(\mathcal{R}_k)_{1 \le k \le K}$, let us assume that $K$ regions, referred to as \textit{current regions}, have already been identified, meaning that the current tree has $K$ leaves. 
As in standard regression trees, each current region $\mathcal{R}_k, \, 1 \le k \le K$ can be decomposed into two sub-regions with respect to a coordinate $1 \le j \le p$ and a splitting point $s_k^j$ that minimizes Eq.~\eqref{emp-risk}. Each split updates $\B{P}$, which now belongs to $M_{n,K+1}(\mathbb{R})$, and $\boldsymbol{\gamma}$, which now belongs to $\mathbb{R}^{K+1}$. Substituting $\boldsymbol{\gamma}$ with its value, the best split for the current region $\mathcal{R}_k$ is given by:
\begin{align}\label{min-split}
\underset{1\leq j \leq p, s \in \mathcal{S}_k^j}{\operatorname{argmin}} \sum_{i=1}^n \left(y^{(i)} - \sum_{\ell=1}^{K+1} \left(\left(\boldsymbol{P}^T \boldsymbol{P}\right)^{-1} \boldsymbol{P}^T \boldsymbol{y}\right)^{(\ell)} {P}_{il}\right)^2,
\end{align}
where $\mathcal{S}_k^j$ denotes the set of splitting points for region $\mathcal{R}_k$ and variable $j$ (more precisely, $\mathcal{S}_k^j$ is the set of middle points of the observations from $\mathcal{R}_k$ projected on the $j$th coordinate). The variable and associated splitting point, which minimizes Eq.~\eqref{min-split}, is then selected to grow the tree, with the number of regions increasing by one.

Lastly, the vector $\boldsymbol{\sigma}$ can either be based on \textit{a priori} knowledge or be learned through a grid search on a validation set. We rely on the latter in our experiments.

\begin{algorithm}[t]
	\caption{Probabilistic Regression Tree}
	\label{algo:PRT}
	\begin{algorithmic}
		\STATE \textbf{Input:} a training set $ {\cal D}_n=\left\{(\B{x}^{(i)}, y^{(i)})_{1\leq i \leq n}\right\}$ and the new covariate $\B{x}$, a vector $\boldsymbol{\sigma}$; 
		\STATE Initialize: $\mathbf{P} = \mathbf{1}_n$
		\REPEAT 
		\STATE Find the best splitting value satisfying \eqref{min-split}, which defines the new set of regions $(\mathcal{R}_k)_k$
		\STATE Update $\mathbf{P}$ with the new set of regions
		\STATE Update $\boldsymbol{\gamma} = \left(\boldsymbol{P}^T \boldsymbol{P}\right)^{-1} \boldsymbol{P}^T \boldsymbol{y}$
		\UNTIL the stopping criterion is met
\STATE\textbf{Return:} the prediction given by Eq. \eqref{treeProbabilistic}.		
	\end{algorithmic}
\end{algorithm}

The algorithm is summarized in Algorithm \ref{algo:PRT}.

Regarding the bias and variance of this estimator, considering a family of functions $\Psi$ reduces the bias with respect to the standard regression tree, because varying the variance vector $\boldsymbol{\sigma}$ allows a different type of regression function and not only a piecewise constant. This also reduces the variance with respect to the standard regression tree due to the smoothness used instead of an indicator function. This is illustrated numerically in Section \ref{exp:bias_var}.

\subsection{Consistency of probabilistic regression trees}
In this section, we recall the main results regarding the consistency of PR trees. First, we introduce important notations and the set of assumptions required to establish consistency.
\paragraph{Notations} Consider the Sobolev space of functions defined, for $s>0$, as: 
\begin{align*}
 H^s(\mathbb{R}^p) &= \{f\in L^2(\mathbb{R}^p, (1+\|\cdot\|_2^2)^{s/2} |\mathcal{F}f(\cdot)| \in L^2(\mathbb{R}^p)\},\\
 H^s([0,1]^p) &= \{f \in L^2([0,1]^p), \exists g \in H^s(\mathbb{R}^p) \text{ s.t. } f = g|_{[0,1]^p}\},
\end{align*} 
where $\mathcal F f$ denotes the Fourier transform of any function $f$, equipped with the following norms:
\begin{align*}
\|h\|_{H^s(\mathbb{R}^p)}&= \int_{\mathbb{R}^p} (1+\|\xi\|_2^2)^{s/2} |\mathcal{F}f(\xi)|^2 d\xi ,\\
 \|h\|_{H^s([0,1]^p)}&=\inf\{\|h\|_{H^s(\mathbb{R}^p)},\,h\in H^s(\mathbb{R}^p),\,\mbox{ s.t. }h=g|_{[0,1]^p}\}.
\end{align*}

\begin{assumption}\label{ass:Hs}
We assume that for some $s \in (1,2)$, $\mathbb{E}(Y|\mathbf{X} = \cdot) \in H^s([0,1]^p)$.
\end{assumption}

\begin{assumption}\label{ass:phi}
 We assume here that $\phi$ satisfies the following conditions: the support of its Fourier transform is $\mathbb{R}^p$, there exists $ r>0$ such that: 
$$ \sup_{\mathbf{v}\in \mathbb{R}^p} |\mathbf{v}|^{1+r+p/2}|\phi(\mathbf{v})| < \infty$$
which ensures that $\phi \in L^2(\mathbb{R}^p, \mathbb{R})$. 
In addition, we also assume that $\phi_{\boldsymbol\sigma}\in B^{(1)}_{1,\infty}$.
\end{assumption}

Set $M>0$, we denote: 
\begin{equation}\label{eq:mathcal-M}
\mathcal{M}(M):=\sup_{|\mathbf v|\geq M}|\phi_{\mathbf \sigma}(\mathbf v)|.
\end{equation}

We also make standard assumptions \citep{scornet:2015} on the regions defined by the tree.
\begin{assumption}
\label{ass:tree}
Set $M>0$. 
Let $(K_n)_{n \in \mathbb{N}}$ such that: 
\begin{align*}
 K_n \underset{n\rightarrow +\infty}{\longrightarrow} &+\infty\\
 K_n (\log n)^9/n \underset{n\rightarrow +\infty}{\longrightarrow} &0\\
 \max_{k=1,\ldots, K_n} [\text{diam}(\mathcal{R}_k \cap [-M,M]^p)] \underset{n\rightarrow +\infty}{\longrightarrow} &0.
\end{align*}
\end{assumption}
The assumption about the diameter of the regions is reasonable for data points lying in a compact subspace: as the number of regions increases to infinity in line with the number of data points, their diameter will decrease.

\paragraph{Consistency} 
We show here that the parameter $\hat\Theta_n$ of PR tree learned from a training set of size $n$ is consistent.
\begin{theorem}\label{thm:PRTree}
We assume that Assumptions \ref{ass:Xcompact}, \ref{ass:Hs}, \ref{ass:phi}, \ref{ass:tree} hold. 
Then
$$ \lim_{n\rightarrow +\infty} \mathbb{E}[|{f}_{\text{PR}}(\mathbf{X};\hat\Theta_n) - \mathbb{E}(Y|\mathbf{X})|^2] = 0.$$
\end{theorem}

This theorem is proven in \cite{Alkhoury_2020}. 

\section{Extension to bagging and boosting ensemble methods}\label{sec:bag-boost}
\subsection{Probabilistic regression random forest}
Bagging consists of averaging simple, noisy, but unbiased models. The RF is a substantial modification of bagging, which builds a collection of decorrelated trees and then averages them. 
The extension of PR trees to RFs over $m$ trees, denoted by PR-RF, is defined as follows:
\begin{align}\label{pred_RF}
f_{\text{PR-RF}}^{(m)} \left(\B{x};  \boldsymbol{\Theta}\right) = \frac1m \sum_{\ell=1}^m {f}_{\text{PR}} \left(\B{x}; \Theta^{(\ell)}\right) 
&=\frac1m \sum_{\ell=1}^m\sum_{k=1}^K \gamma_k^{(\ell)} \Psi(\B{x};\mathcal{R}_{k}^{(\ell)},\boldsymbol{\sigma})
\end{align}
where {$ \boldsymbol{\Theta} = (\Theta^{(1)},\ldots, \Theta^{(m)})$} 
with $\Theta^{(\ell)}$ characterizing the $\ell^{th}$ RF's tree in terms of parameters (i.e., split variables, cut points, predictions, and variances). 
The corresponding algorithm is summarized in Algorithm \ref{algo:PRF}. Note that bagging can be undertaken for any extension of regression trees. 

The standard RF is known to reduce the variance, because it averages identically distributed random variables (each tree), which are constructed to be the least correlated using bootstrap variables for each tree. This conclusion applies to any bagging extension, and specifically to the PR-RF, with the bias-variance trade-off being illustrated in Section \ref{exp:bias_var}.

\begin{algorithm}[t]
	\caption{Probabilistic Random Forest (PR-RF)}
	\label{algo:PRF}
	\begin{algorithmic}
		\STATE \textbf{Input:} a training set $ {\cal D}_n=\left\{(\B{x}^{(i)}, y^{(i)})_{1\leq i \leq n}\right\}$, the new covariate $\B{x}$, a vector $\boldsymbol{\sigma}$, and the number of trees $m$; 
		\STATE Initialize: $\mathbf{P}^{(\ell)}=\mathbf{1}_n$ for $\ell=1,\ldots, m$
		\FOR{$\ell=1,\ldots, m$}
		\STATE Randomly select variables to be used in the tree $\ell$
		\REPEAT 
		\STATE Find the best splitting value satisfying \eqref{min-split} from the restricted set of variables that defines the new set of regions $(\mathcal{R}_k^{(\ell)})_k$
		\STATE Update $\mathbf{P}^{(\ell)}$ with the new set of regions
		\STATE Update $\boldsymbol{\gamma}^{(\ell)}=\left((\boldsymbol{P}^{(\ell)})^T \boldsymbol{P}^{(\ell)}\right)^{-1} (\boldsymbol{P}^{(\ell)})^T \boldsymbol{y}$
		\UNTIL the stopping criterion is met
		\ENDFOR
\STATE\textbf{Return:} the prediction given by Eq. \eqref{pred_RF}		
	\end{algorithmic}
\end{algorithm}

From a theoretical viewpoint, we can easily adapt the consistency from \cite{scornet:2015} to derive the consistency of PR-RF when the number of trees $m$ grows to infinity. Two cases are described in this paper depending on the condition to construct each tree: either the total number of leaves in each tree tends to infinity more slowly than the number of selected data points or the trees are fully grown (i.e., the number of leaves is equal to the number of selected data points). We derive our theorem from the first case, but its extension to the second case is straightforward. As stronger assumptions can lead to strong consistency, we keep the weakest case here. Note that we do not need the assumption about the additive models made in \cite{scornet:2015} as we consider Sobolev spaces for the regression function.

\begin{theorem}
For $n \in \mathbb{N}$ as the sample size, let $t_n$ be the total number of leaves in each tree, and $a_n$ be the number of selected data points to construct each tree. 
We also assume that the assumptions \ref{ass:Xcompact}, \ref{ass:Hs}, \ref{ass:phi}, \ref{ass:tree} hold.
Then, provided $a_n \rightarrow +\infty$, $t_n \rightarrow +\infty$, $t_n(\log a_n)^9/a_n \rightarrow 0$, the PR-RF are consistent, that is: 
\[
\lim_{n\rightarrow +\infty, m\rightarrow +\infty} \mathbb{E}[|f_{\text{PR-RF}}(\mathbf{X};\hat\Theta_n) - \mathbb{E}(Y|\mathbf{X})|^2] = 0.
\]
\end{theorem}
\begin{proof}\textbf{[sketch]} From Theorem \ref{thm:PRTree}, we deduce the existence of $\hat\Theta_n$ independent of $\mathbf{X}$ such that:
\begin{equation}\label{eq:step1a}
\mathbb{E}_{\mathbf X}[\left|f_{\text{PR}}(\mathbf{X};\hat\Theta_n)-\mathbb{E}[Y|\mathbf X]\right|^2]\leq c_n
\end{equation}
with $c_n\to 0$ as $n\to \infty$. We first prove that: 
\begin{equation}\label{eq:step1b}
\mathbb{E}_{\mathbf X}[\left|\mathbb{E}_{\mathcal{D}_n}[f_{\text{PR}}(\mathbf X;\hat\Theta_n)]-\mathbb{E}[Y|\mathbf X]\right|^2]\to 0\mbox{ as }n\to \infty.
\end{equation}
First, observe that by independence of $\hat\Theta_n$ and $\mathbf X$: 
\[
\mathbb{E}_{\mathcal{D}_n}\left[f_{\text{PR}}(\mathbf X;\hat\Theta_n)\right]-\mathbb{E}[Y|\mathbf X]=\mathbb{E}_{\mathcal{D}_n}\left[f_{\text{PR}}(\mathbf X;\hat\Theta_n)-\mathbb{E}[Y|\mathbf X]\right]
\]
Hence:
\begin{eqnarray*}
\left|\mathbb{E}_{\mathcal{D}_n}[f_{\text{PR}}(\mathbf X;\hat\Theta_n)]-\mathbb{E}[Y|\mathbf X]\right|^2&=&\left|\mathbb{E}_{\mathcal{D}_n}\left[f_{\text{PR}}(\mathbf X;\hat\Theta_n)-\mathbb{E}[Y|\mathbf X]\right]\right|^2\\
&\leq&\mathbb{E}_{\mathcal{D}_n}[\left|f_{\text{PR}}(\mathbf X;\hat\Theta_n)-\mathbb{E}[Y|\mathbf X]\right|^2]
\end{eqnarray*}
where in the last display, we use the Cauchy-Schwarz inequality. Hence, by monotonicity of $\mathbb{E}_{\mathbf X}$:
\begin{eqnarray*}
\mathbb{E}_{\mathbf X}[\left|\mathbb{E}_{\mathcal{D}_n}[f_{\text{PR}}(\mathbf X;\hat\Theta_n)]-\mathbb{E}[Y|\mathbf X]\right|^2]&\leq&\mathbb{E}_{\mathbf X}[\mathbb{E}_{\mathcal{D}_n}[\left|f_{\text{PR}}(\mathbf X;\hat\Theta_n)-\mathbb{E}[Y|\mathbf X]\right|^2]]\\
&=&\mathbb{E}_{\mathcal{D}_n}[\mathbb{E}_{\mathbf X}[\left|f_{\text{PR}}(\mathbf X;\hat\Theta_n)-\mathbb{E}[Y|\mathbf X]\right|^2]]\\
&\leq& \mathbb{E}_{\mathcal{D}_n}[c_n]=c_n
\end{eqnarray*}
with the two last displays stemming from the independence of $\hat\Theta_n$ and $\mathbf X$ and~\eqref{eq:step1a}. Since $c_n\to 0$ as $n\to\infty$, this directly implies~\eqref{eq:step1b}. 

Let us now consider $\hat\Theta^{(1)}_n,\cdots,\hat\Theta^{(m)}_n$ iid copies of $\hat\Theta_n$, jointly independent of $\mathbf X$. From the law of large numbers, almost surely then:
\[
 \lim_{m\rightarrow+\infty} {f}_{\text{PR-RF}}^{(m)}\left(\B{X}; \hat{ \boldsymbol{\Theta}}_n \right) 
 = \lim_{m\rightarrow+\infty} \frac1m \sum_{\ell=1}^m {f}_{\text{PR}} \left(\B{X}; \hat\Theta_n^{(\ell)}\right) 
 = \mathbb{E}_{\mathcal{D}_n} (f_{\text{PR}}(\mathbf{X};\hat\Theta_n)).
\]
\begin{equation}\label{eq:step2}
\end{equation}
Combining~\eqref{eq:step1b} and~\eqref{eq:step2} using Fubini's theorem, it yields:
\[\lim_{n\rightarrow +\infty, m\rightarrow +\infty} \mathbb{E}[|{f}_{\text{PR-RF}}(\mathbf{X};\hat\Theta_n) - \mathbb{E}(Y|\mathbf{X})|^2] = 0.\]
\end{proof}

Note that assumptions $1$ to $4$ are only needed to establish the result presented in Theorem~\ref{thm:PRTree} so that the above proof, derived from \citet{scornet:2015}, is valid for any approach consistent in the sense given in Theorem~\ref{thm:PRTree}. 

\subsection{Probabilistic regression gradient boosted trees}

Gradient-boosted trees gradually improve the prediction by optimizing the residuals with respect to the prediction based on the trees constructed so far. Let us assume that $(m-1)$ PR trees have been built so far. The $m^{th}$ smooth tree and its parameter $\hat{\Theta}^{(m)}$ are obtained from Eq.~\eqref{emp-risk} using:
\begin{align}
\underset{{\Theta^{(m)}}}{\operatorname{argmin}} \sum_{i=1}^n \left(\left(y^{(i)} - \sum_{\ell=1}^{m-1} \sum_{k=1}^{K^{(\ell)}} \gamma_k^{(\ell)} [P_{ik}]^{(\ell)}\right) - \sum_{k=1}^{K^{(m)}} \gamma_k^{(m)} [P_{ik}]^{(m)}\right)^2\nonumber
\end{align}
where the matrices $\mathbf{P}$ depend on the regions.
So, the prediction function is given by:
\begin{align}\label{pred_GBT}
{f}_{\text{PR-GBT}}^{(m)} \left(\B{x}; \boldsymbol{\Theta}\right) = \sum_{\ell=1}^m {f}_{\text{PR}} \left(\B{x}; \Theta^{[\ell]}\right) 
&= \sum_{\ell=1}^m \sum_{k=1}^K \gamma_k^{(\ell)} \Psi(\B{x};\mathcal{R}_{k}^{(\ell)},\boldsymbol{\sigma})
\end{align}
where $\boldsymbol{\Theta} = (\Theta^{[1]},\ldots, \Theta^{[m]})$. 

\begin{algorithm}[t]
	\caption{ Probabilistic Regression Gradient Boosted Tree (PR-GBT)}
	\label{algo:P-GBT}
	\begin{algorithmic}
			\STATE \textbf{Input:} a training set $ {\cal D}_n=\left\{(\B{x}^{(i)}, y^{(i)})_{1\leq i \leq n}\right\}$, the new covariate $\B{x}$, a vector $\boldsymbol{\sigma}$, and the number of trees $m$; 
		\STATE Initialize: $\mathbf{P}^{(\ell)}=\mathbf{1}_n$ for $\ell=1,\ldots, m$
		\FOR{$\ell=1,\ldots, m$}
		\REPEAT 
		\STATE Find the best splitting value satisfying \eqref{min-split}, where $y$ is updated by the residuals, which defines the new set of regions $(\mathcal{R}_k^{(\ell)})_k$
		\STATE Update $\mathbf{P}^{(\ell)}$ with the new set of regions
		\STATE Update $\boldsymbol{\gamma}^{(\ell)}=\left((\boldsymbol{P}^{(\ell)})^T \boldsymbol{P}^{(\ell)}\right)^{-1} (\boldsymbol{P}^{(\ell)})^T \boldsymbol{y}$
		\UNTIL the stopping criterion is met
		\ENDFOR
\STATE\textbf{Return:} the prediction given by Eq. \eqref{pred_GBT}		
	\end{algorithmic}
\end{algorithm}

Algorithm \ref{algo:P-GBT} details the different steps for obtaining probabilistic regression gradient-boosted trees (PR-GBT). Note that boosting methods are known to reduce the bias of the prediction function (while allowing for a small variance), which is true for our PR-GBT prediction function, as illustrated in Section \ref{exp:bias_var}.

However, boosting forever can overfit the data, making it necessary to stop the procedure with an adaptive finite number of steps \citep{ZhangYu2005}. In our theoretical analysis, as we mimic the results obtained in \cite{ZhangYu2005}, we apply early stopping; in practice however, we fix the number of trees. Furthermore, in line with \cite{ZhangYu2005}, we consider here convergence in probability (and not in $L_2$ as before).

\begin{theorem}
For the sample size $n$, consider $k_n$ and $\alpha_n$ as two sequences of numbers (independent of the sample) such that $\lim_{n \rightarrow +\infty} k_n = \infty, \lim_{n\rightarrow +\infty} \alpha_n^2 \sqrt{\log(n)} = 0$. 
We also assume Assumptions \ref{ass:Xcompact}, \ref{ass:Hs}, \ref{ass:phi}.
Then, as long as we stop at step $\hat m$ based on the observations such that $\hat m \geq k_n$ and $\|\hat{f}_{\text{PR-GBT}}^{\hat m}\|_1 \leq \alpha_n$, PR-GBT are consistent: 
$$ {f}_{\text{PR-GBT}}^{\hat m} \left(\B{X}; \hat{ \boldsymbol{\Theta}}_n\right) \overset{\mathbf{P}}{\underset{n\rightarrow +\infty}{\longrightarrow}} \mathbb{E}[Y|\mathbf X].$$
\end{theorem}
\begin{proof}\textbf{sketch}
We use Theorem 3.1 from \cite{ZhangYu2005} to obtain this theorem.
The least square function is used as the loss function. 
The main idea of this proof is to decompose the risk into three terms, bounded independently. Let $f^{\star}$ be a candidate with the form \eqref{pred_GBT} constructed by boosting. 
\begin{align}
 & \mathbf{E}_{\mathcal{D}_n} \mathbf{E}_{\mathbf{X}, Y}\left(Y - {f}_{\text{PR-GBT}}^{\hat m} \left(\B{X}; \hat{ \boldsymbol{\Theta}}_n\right)\right)^2 - \mathbf{E}_{\mathbf{X}, Y}\left(Y - {f}^{\star} \left(\B{X}\right)\right)^2 \nonumber\\
 &= \mathbf{E}_{\mathcal{D}_n} \mathbf{E}_{\mathbf{X}, Y}\left(Y - {f}_{\text{PR-GBT}}^{\hat m} \left(\B{X}; \hat{ \boldsymbol{\Theta}}_n\right)\right)^2 - \frac1n 
 \mathbf{E}_{\mathcal{D}_n} \sum_{i=1}^n \left(y^{(i)} - {f}_{\text{PR-GBT}}^{\hat m} \left(\B{x}^{(i)}; \hat{ \boldsymbol{\Theta}}_n\right)\right)^2 \label{term1}\\
 &+ \frac1n 
 \mathbf{E}_{\mathcal{D}_n} \sum_{i=1}^n \left(y^{(i)} - {f}^{\star} \left(\B{x}^{(i)}\right)\right)^2
 - \mathbf{E}_{\mathbf{X}, Y}\left(Y - {f}^{\star} \left(\B{X}\right)\right)^2\label{term2}\\
 &+ \frac1n 
 \mathbf{E}_{\mathcal{D}_n} \sum_{i=1}^n \left(y^{(i)} - {f}_{\text{PR-GBT}}^{\hat m} \left(\B{x}^{(i)}; \hat{ \boldsymbol{\Theta}}_n\right)\right)^2
 - \frac1n 
 \mathbf{E}_{\mathcal{D}_n} \sum_{i=1}^n \left(y^{(i)} - {f}^{\star} \left(\B{x}^{(i)}\right)\right)^2. \label{term3}
\end{align}

The first term \eqref{term1} can be controlled using Rademacher complexity and converges to 0 when $n \rightarrow +\infty$. \sout{PR trees have a bounded Vapnik-Chervonenkis (VC) dimension, as an extension through a variance term of standard regression trees.} A central limit theorem helps bound the second term \eqref{term2} by a term proportional to $\frac{1}{\sqrt{n}}$, up to a constant. 
Lastly, the third term \eqref{term3} can be bounded by a term which converges to 0 when $n \rightarrow +\infty$ by numerical convergence analysis (Corollary 4.3 from \cite{ZhangYu2005}), so that ${f}_{\text{PR-GBT}}^{\hat m} \left(\B{X}; \hat{ \boldsymbol{\Theta}}_n\right) \overset{\mathbf{P}}{\underset{n\rightarrow +\infty}{\longrightarrow}} \mathbb{E}[Y|\mathbf X]$.
\end{proof}

Note that this proof is also generic in the sense that it applies to any regression tree method with bounded VC dimension with dense function space.

\section{Probabilistic regression Bayesian additive regression trees}\label{sec:BART}
The Bayesian additive regression tree is a boosting extension, where \textit{a priori} distribution adds randomness. In this section, we describe how to construct P-BART, an extension of BART \citep{bart} using PR trees, which is also summarized in Algorithm \ref{algo:P-BART}. Surprisingly, BART is not widely used despite its promising performance. So, here we provide further details about the method adapted to PR trees.

\begin{algorithm}[t]
	\caption{Probabilistic Bayesian Additive Regression Tree (P-BART)}
	\label{algo:P-BART}
	\begin{algorithmic}
		\STATE \textbf{Input:} a training set $ {\cal D}_n=\left\{(\B{x}^{(i)}, y^{(i)})_{1\leq i \leq n}\right\}$ and the new covariate $\B{x}$; 
		$m$ the number of trees, $\text{it}_{\text{burn}}$ the number of iterations in the burn in the period, $\text{it}_{\max}$ the maximal number of iterations, $(\lambda, \nu)$ parameters for the distribution of the noise variance, $\sigma_\gamma^2$ the parameter for the prior distribution of $\boldsymbol{\gamma}$
\STATE Initialization: $m$ trees $(T^{(\ell),0})_{1\leq \ell \leq m}$ with one terminal node, $\gamma_0^{(\ell)} \sim \mathcal{N}(0, \sigma_\gamma^2)$, and $\text{it}= 0$.
		\STATE Draw $\tilde \sigma \sim \mathcal{IG}(\nu/2, \lambda \nu/2)$
		\REPEAT
		\STATE $\text{it}= \text{it} + 1$
	\FOR{ $\ell=1,\ldots, m$}
		\STATE Generate a new tree $T^{(\ell),\text{it}}_*$ from $T^{(\ell),\text{it-1}}$
		\STATE Define $T^{(\ell),\text{it}}$ by $T^{(\ell),\text{it}}_*$ or by $T^{(\ell),\text{it-1}}$ using a Metropolis-Hasting algorithm with the probability given in Eq. \ref{eq:proba_MH} (based on Prop. \ref{prop:lik2})
		\STATE Compute $(\boldsymbol \gamma^{(\ell), \text{it}})_{ 1\leq \ell \leq m}$ using Prop. \ref{prop:posteriori:gamma2}
		\ENDFOR
		\STATE Draw $\tilde \sigma$ from Eq. \eqref{eq:posteriori_sigma}
		\UNTIL{ $ \text{it}_{\text{burn}} < \text{it} < \text{it}_{\max}$ and the trees are changing}

\STATE\textbf{Return:} the prediction given by Eq. \eqref{eq:pred_P-BART}		
	\end{algorithmic}
\end{algorithm}

We consider $m$ distinct regression trees, with the $\ell^{th}$ tree having a tree structure ${T}^{(\ell)}$ and weights $\boldsymbol\gamma^{(\ell)}=(\gamma_1^{(\ell)},\ldots, \gamma_K^{(\ell)})$. Trees are fitted iteratively until no change is observed (or a maximum number of iterations is reached), thus holding all other $m-1$ trees constant and considering the residual response that remains unfitted. 

The prediction using P-BART is made as an averaging over the iterations (after burning) of the sum (boosting part) of the prediction for a new covariate $\B{x}$ through a PR tree, as given in Eq. \eqref{treeProbabilistic}.
This leads to the following formula, where we denote as $T^{(\ell),t}$ the $j$th tree at iteration $t$, and similarly its parameters, while $\text{it}$ and $\text{it}_{\text{brun}}$ are respectively the number of iterations performed and the number of iterations for the burning:
\begin{align}\label{eq:pred_P-BART}
{f}_{\text{P-BART}}(\B{x}; \Theta) = \frac{1}{\text{it} - \text{it}_{\text{burn}}} \sum_{t = \text{it}_{\text{burn}+1}}^{\text{it}} \sum_{\ell=1}^m \sum_{k=1}^{K^{(\ell)}} 
\Psi(\B{x};\mathcal{R}_{k}^{(\ell),t},\boldsymbol{\sigma}) \gamma_k^{(\ell),t}
\end{align}
where $\Theta$ corresponds to all the parameters needed to define P-BART.

To do so, we define the prior distributions and deduce the posterior distributions after observing the data. We present here a method to approximate the posterior distributions. 

\subsection{Prior distributions}
\begin{align*}
\mathbb P( T^{(1)},\boldsymbol{\gamma}^{(1)}, \ldots, T^{(m)}, \boldsymbol{\gamma}^{(m)}, \tilde\sigma) &= \mathbb P( T^{(1)},\boldsymbol{\gamma}^{(1)}, \ldots, T^{(m)}, \boldsymbol{\gamma}^{(m)}) \mathbb P( \tilde\sigma) \\
&= \mathbb P( \tilde\sigma) \prod_{\ell=1}^m \mathbb P(T^{(\ell)},\boldsymbol{\gamma}^{(\ell)}) = \mathbb P( \tilde\sigma) \prod_{\ell=1}^m \mathbb P(\boldsymbol{\gamma}^{(\ell)}|T^{(\ell)}) \mathbb P(T^{(\ell)}).
\end{align*}
Contrary to the standard BART model, the weights are no longer independent of each other here.

Following the standard BART model, we assume that the prior distribution on the prediction weights for each region is normal and centered (which renormalizes the observations as with standard BART):
\[
\gamma_k^{(\ell)} | \sigma_{\gamma} \sim \mathcal{N}( 0, \sigma_\mu^2)
\]
The prior distribution on the noise variance is inverse Gamma: 
$$\tilde\sigma^2 \sim \mathcal{IG}(\nu/2, \lambda \nu/2)$$
The prior distribution on the tree structure $\mathbb P(T^{(\ell)})$ is defined by three aspects: the probability that a node at depth $d$ would split, which is given by $\alpha/(1+d)^\beta$; the distribution used to select the covariate to split an internal node, which is uniform or Dirichlet for the variable selection (see \cite{linero2018bayesian}); and the distribution used to select the cut-off point in an internal node once the covariate is selected, which is uniform.
Note that we can compute the probability matrices $\mathbf P^{(\ell)}$ from the training set and the tree structure. 

\subsection{Approximation of the posterior distribution}
We can now compute the posterior distribution:
\begin{align}
\mathbb P((T^{(\ell)}, \boldsymbol \gamma^{(\ell)})_{1\leq \ell \leq m}, \tilde \sigma | Y) \propto \mathbb P(Y| (T^{(\ell)}, \boldsymbol \gamma^{(\ell)})_{1\leq \ell \leq m}, \tilde \sigma ) \mathbb P((T^{(\ell)}, \boldsymbol \gamma^{(\ell)})_{1\leq \ell \leq m}, \tilde \sigma).
\label{eq:full_posterior}
\end{align}

It can be simplified into two major posterior draws using Gibbs sampling. Draw 
$(T^{(\ell)}, \boldsymbol \gamma^{(\ell)}) | (T^{-\ell}, \boldsymbol{\gamma}^{-\ell}, Y, \tilde \sigma)$ 
for $j=1, \ldots, m$, where $T^{-\ell}$ (resp. $\boldsymbol{\gamma}^{-\ell}$) denotes all the tree structures (resp. the weights) except for that associated with the $\ell$th tree, and then 
draw a new variance noise $\tilde\sigma^2$ from the associated posterior distribution $\tilde \sigma | (T^{(1)}, \boldsymbol \gamma^{(1)}), \ldots, (T^{(m)}, \boldsymbol \gamma^{(m)}), Y$, which is given by
\begin{align}
\label{eq:posteriori_sigma}
\tilde \sigma | (T^{(\ell)}, \boldsymbol \gamma^{(\ell)})_{1\leq \ell \leq m}, Y \sim \mathcal{IG}\left(\frac{\nu + n}{2}, \frac{\nu \lambda + \sum_{i=1}^n (Y^{(i)} - \sum_{\ell=1}^m \hat{f}_{\text{PR}}(\B{X}^{(i)}; T^{(\ell)}, \boldsymbol \gamma^{(\ell)}))^2}{2}\right),
\end{align}
where we use a slight abuse of notation for $f_{\text{PR}}$, where the regions are defined by the topology of the tree. 
Although the second step is clear, we need to give more details for the first step.

To do so, as in \cite{bart_2019} with the use of boosting-based methods, we denote as $R^{(\ell)}$ the random variable corresponding to the residuals of the $m-1$ regression sum of trees fit excluding the $\ell$th tree, and $R^{(\ell),(i)}, 1 \le i \le n$ the contribution of the $i^{th}$ example to $R^{(\ell)}$.
Let $\hat{f}_{\text{PR}}(\B{X}^{(i)}; T^{(\ell)}, \boldsymbol{\gamma}^{(\ell)})$ denote the prediction of the $\ell$th tree on the $i$th example. This leads to:
\begin{align*}
\hat{f}_{\text{PR}}(\B{X}^{(i)}; T^{(\ell)}, \boldsymbol{\gamma}^{(\ell)}) &= \sum_{k=1}^K \gamma_{k}^{(\ell)} \underbrace{\Psi(\B{X}^{(i)} ; \mathcal{R}_k^{(\ell)},\boldsymbol{\sigma})}_{P^{(\ell)}_{ik}},\\
R^{(\ell),(i)}&= Y^{(i)} - \sum_{l' \neq l} \hat{f}_{\text{PR}}(\B{X}^{(i)}; T^{\ell'}, \boldsymbol{\gamma}^{\ell'}) = \hat{f}_{\text{PR}}(\B{X}^{(i)}; T^{(\ell)}, \boldsymbol{\gamma}^{(\ell)}) + \varepsilon^{(i)},
\end{align*}
such that from the assumed general noise model (Eq.~\ref{eq:reg}):
\[
R^{(\ell),(i)}| \mathbf{X}^{(i)}, Y^{(i)}, T^{(\ell)}, \boldsymbol{\gamma}^{(\ell)}, \tilde\sigma = 
R^{(\ell),(i)}| \hat{f}_{\text{PR}}(\B{X}^{(i)}; T^{(\ell)}, \boldsymbol{\gamma}^{(\ell)}) , \sigma \sim \mathcal{N}\left(\sum_{k=1}^K \gamma_k^{(\ell)} P^{(\ell)}_{ik}, \tilde\sigma^2\right).
\]

Note that: 
\begin{align*}
\mathbb{P} (T^{(\ell)}, \boldsymbol \gamma^{(\ell)} | T^{-\ell}, \boldsymbol{\gamma}^{-\ell}, Y, \tilde \sigma) &= \mathbb P (T^{(\ell)}, \boldsymbol \gamma^{(\ell)} | R^{(\ell),.}, \tilde\sigma)\\
&= \mathbb P ( \boldsymbol \gamma^{(\ell)} | T^{(\ell)}, R^{(\ell),.}, \tilde\sigma) P ( T^{(\ell)}| R^{(\ell),.}, \tilde\sigma).
\end{align*}
We can then use a Metropolis-Hasting algorithm: first, we generate a candidate tree $T_*^{(\ell)}$ from the previous tree $T^{(\ell)}$ using four local steps: grow, prune, swap (i.e., splitting criteria of two non-terminal nodes), and change (i.e., splitting criteria of a single non-terminal node). 
This new tree is accepted with the probability:
\begin{align}
\label{eq:proba_MH}
\alpha(T^{(\ell)}, T_*^{(\ell)}) = \min\left\{ 1, \frac{q(T_*^{(\ell)}, T^{(\ell)})}{q(T^{(\ell)}, T_*^{(\ell)})} 
\frac{\mathbb{P}(R^{(\ell),.} | X, T_*^{(\ell)}, \boldsymbol \gamma^{(\ell)})}{\mathbb{P}(R^{(\ell),.} | X, T^{(\ell)}, \boldsymbol \gamma^{(\ell)})}
\frac{\mathbb{P}(T_*^{(\ell)})}{\mathbb{P}(T^{(\ell)})} \right\},
\end{align}
where $q(T^{(\ell)}, T_*^{(\ell)})$ is the probability of moving from the old to the new tree. We need to compute the posterior distribution of $R^{(\ell),.}|T^{(\ell)},\boldsymbol{\gamma}^{(\ell)},\sigma^2$, which is derived in the following proposition. 
Unlike in the standard case where all the weights and residuals across regions are independent, one cannot solely focus on the region of interest. Instead, it is necessary to compute the posterior of the residuals for the complete tree and compare the ratio of the posteriors. 

\begin{proposition}[A posteriori distribution of $R^{(\ell),.}|T^{(\ell)},\boldsymbol{\gamma}^{(\ell)},\Sigma_K$] \label{prop:lik2}
From the above development, denoting as $\phi_{\mathbf{N}}$ the density of the normal distribution, one finds:
\begin{align*}
\mathbb P(R^{(\ell),.}|T^{(\ell)},\boldsymbol{\gamma}^{(\ell)},\sigma_{\gamma},\tilde \sigma) = \phi_{\mathbf{N}}(R^{(\ell),.}; \mathbf{0}, \Sigma_0^{(\ell)}),
\end{align*}
where $\Sigma_K^{(\ell)}=\sigma^2 \mathbb{I}_n$, and $(\Sigma_0^{(\ell)})^{-1}$ is obtained using the following recursive formula, $\forall \ell, \, 1 \le \ell \le K$:
\[
\Sigma_{K-\ell}^{-1} = \Sigma_{K+1-\ell}^{-1} - \frac{\Sigma_{K+1-\ell}^{-1} c_{.(K+1-\ell)}c_{.(K+1-\ell)}^T \Sigma_{K+1-\ell}^{-1}}{c_{.(K+1-\ell)}^T \Sigma_{K+1-\ell}^{-1} c_{.(K+1-\ell)} + \frac{1}{\sigma_\mu^2}}
\]
where $c_{i,k} = P_{ik}^{(\ell)}$.
\end{proposition}

\noindent The ratios for changing, growing, and pruning operations are directly computed from the above formulas by comparing the posterior probability of the residuals of the new tree to that of the original.
All derivations are close to that of \cite{bart_2019} and are thus not included in this paper. Note that in particular, one has: 
$$ \det(\Sigma_0) = \left(\prod_{\ell=0}^{K-1} (c_{.(K-\ell)}^T \Sigma_{K-\ell}^{-1} c_{.(K-\ell)} + \frac{1}{\sigma_\mu^2})\right) \sigma^K_{\mu} \sigma^n$$
and:
\begin{align*}
\log(\mathbb P(R^{(\ell)}|T^{(\ell)},\boldsymbol{\gamma}^{(\ell)},\sigma) ) & = -\frac{n}{2} \log (2\pi) - \frac12 \log( \det(\Sigma_0) ) -\frac{1}{2} R_{\ell}^T \Sigma_0^{-1} R_{\ell}.
\end{align*}

If the tree is accepted, update the associated weights $\boldsymbol\gamma^{(\ell)}$, which are drawn from the posterior distribution of $\gamma_{k}^{(\ell)} | T^{(\ell)}, R^{(\ell),.}, \boldsymbol \gamma^{(\ell)}_{-k}, \sigma_{\gamma}, \tilde\sigma$ established in the following proposition, which follows standard derivations as in~\cite{bart_2019}. However, the weights are not independent of each other, so one needs to condition all the other weights (for all the other regions) when computing the posterior distribution of a given weight.

\begin{proposition}[A posteriori distribution of $\gamma_k^{(\ell)}$] \label{prop:posteriori:gamma2}
Let $\boldsymbol \gamma^{(\ell)}_{-k}$ denote all the weights of the $j^{th}$ tree except for that associated with the $k^{th}$ region. Then:
\begin{align*}
\gamma_{k}^{(\ell)} | T^{(\ell)}, R^{(\ell),.}, \boldsymbol \gamma^{(\ell)}_{-k},\sigma_{\gamma}, \tilde\sigma \sim \mathcal{N} \left(\frac{\sigma^2_{\gamma} B_k^{(\ell)}}{\tilde\sigma^2+\sigma^2_{\gamma} A_k^{(\ell)}},\frac{\tilde\sigma^2 \sigma^2_{\gamma}}{\tilde\sigma^2+\sigma^2_{\gamma} A_k^{(\ell)}}\right)
\end{align*}
with $A_k^{(\ell)}=\sum_{i=1}^n (P_{ik}^{(\ell)})^2$ and $B_k^{(\ell)}=\sum_{i=1}^n P_{ik}^{(\ell)}(R^{(\ell),(i)}- \sum_{k' \ne k} \gamma_{k'}^{(\ell)} P_{ik'}^{(\ell)})$.
\end{proposition}

\subsection{Consistency}
In this section, we determine the consistency of PR-BART using the convergence of the posterior distribution to the true function. We first make the following set of assumptions following \cite{linero2018bayesian}.
\begin{assumption} \textit{(set of assumptions, 5.1 to 5.5)} \label{ass:BART}
\begin{enumerate}
\item[A\ref{ass:BART}.1] There are some constants $(C, C_1 , C_2 )$ such that the prior distribution on number of trees
T satisfies $\pi(T = t) \geq C_1 \exp(-C_2 t)$ and $\Pi(T \geq t) \leq \exp(-C t)$ for $t\in\mathbb{N}$.
\item[A\ref{ass:BART}.2] The prior on the splitting variable chosen at each node is a discrete uniform random variable on $\{1,\cdots,p\}$.
\item[A\ref{ass:BART}.3] The leaf coefficients $\gamma_{t,j}$ are iid with density $\pi_{\boldsymbol\gamma}$ where $\pi_{\boldsymbol\gamma}({\boldsymbol\gamma}) \geq B_1 \exp (-B_2 \|{\boldsymbol\gamma}\|_2)$ for all $\mu$, and $\pi_{\boldsymbol\gamma}(|{\boldsymbol\gamma}| \geq t) \geq B_3 \exp (-B_2 t^{B_4})$ for
some positive constants $B_1 , B_2, B_3, B_4$.
\item[A\ref{ass:BART}.4] Split locations $x_{t,j}$ are assumed to be uniform on the range $[0,1]$.
\item[A\ref{ass:BART}.5] Fix $\varepsilon>0$. 
There exists $r_{\phi}>0$ such that:
\begin{equation}\label{eq:ass-beta}
\beta(\varepsilon^{-2/s})\leq \exp(\varepsilon^{-r_\phi}).
\end{equation}
where $\beta$ is defined by:
\begin{equation*}
 \beta(\lambda):=\frac{\lambda^{p/2}}{\inf_{|\boldsymbol \omega|\le \lambda}|\mathcal{F}[\phi_{\boldsymbol\sigma}](\boldsymbol \omega)|}.
 \end{equation*}
\end{enumerate}
\item[A\ref{ass:BART}.6] The prior on the depth $D_t$ of a tree is truncated such that $\Pi(D_t>p_0) = 0$ for some $p_0\geq p$, and $\pi(D_t = k) >0$ for $k = 0,\ldots, 2p$.
\end{assumption}
As in \cite{linero2018bayesian}, Assumption A\ref{ass:BART}.3 can be weakened to allow for the Gaussian distribution with positive Cauchy prior distribution on the standard deviation parameter. Lastly, Assumption~\ref{ass:BART}.5 is useful for controlling the convergence rate for PR trees \citep{Alkhoury_2020}.

{Two central propositions, the proofs of which are given in Appendix~\ref{app}, can be established from the above assumptions, the second one deriving from the first one and corresponding to a concentration result of the prior of the regression function.

\begin{proposition}\label{prop:step1}
We assume that Assumptions \ref{ass:Xcompact}, \ref{ass:Hs}, \ref{ass:phi}, and \ref{ass:BART} hold. Then there exists $T$ partitions $(\mathcal{R}^{(1)}_k),\cdots,(\mathcal{R}^{(T)}_k)$ and $h$ the sum of $T$ trees $(h^1,\cdots,h^T)$ with $T\leq C \varepsilon^{-p}$, with a maximum depth of $C p\log(1/\varepsilon)$ and the total number of nodes $N^t$ for each tree $t$ at most $C\,\log(1/\varepsilon)$ of the form: 
\[
h:=\sum_{t=1}^Th^{(t)}\mbox{ with }h^{(t)}:=\sum_{k=1}^{K^{(t)}} \gamma_{k}^{(t)} \Psi(\cdot;\mathcal{R}^{(t)}_k)\mbox{ with }\|h\|_\infty\leq \beta(\varepsilon)
\]
such that for any $\mathbf{x}\in [0,1]^p$:
\begin{align*}
&|\mathbb{E}[Y|\mathbf{X}=\cdot](\mathbf{x})-h(\mathbf{x})| 
\le \left[C(\boldsymbol\sigma,p, \phi)\|\mathbb{E}[Y|\mathbf{X}=\cdot]\|_{H^s(\mathbb{R}^p)}\right]\cdot\varepsilon
\end{align*}
\end{proposition}
%
%
\begin{proposition}\label{prop:step2}
We assume that Assumptions \ref{ass:Xcompact}, \ref{ass:Hs}, \ref{ass:phi}, and \ref{ass:BART} hold. Then there exists a constant C independent of $(n,p)$ such that for all sufficiently large $n$, the prior $\pi$ over regression function $h$ satisfies the following:
\[
\pi\left[\|h-\mathbb{E}[Y|\mathbf{X}=\cdot]\|_{L^2([0,1]^p)}\leq \varepsilon_n \right]\geq C\exp\left(-n\varepsilon_n^2\right)
\]
where $\varepsilon_n:=[\log\log n]^{-1/r_\phi}$.
\end{proposition}

We are now in a position to state our main consistency result.
\begin{theorem}\label{th:PR-BART-main}
We assume that Assumptions \ref{ass:Xcompact}, \ref{ass:Hs}, \ref{ass:phi}, \ref{ass:tree} and \ref{ass:BART} hold, and in addition $n\varepsilon_n^2\rightarrow \infty$ and $\varepsilon_n\rightarrow 0$ as $n\rightarrow \infty$. Then, for some $M_1>0$,
\[
\Pi_n\left(\|\mathbb{E}[Y|\mathbf{X}=\cdot]-f_{P-BART}\|_n\geq M_1\varepsilon_n^2\right)\rightarrow 0\mbox{ in probability as }n \rightarrow \infty,
\] 
\end{theorem}
The proof of this theorem, detailed in Appendix~\ref{app}, follows the same steps as the proof of Theorem~3 in \cite{linero2018bayesian} using \cite{bhattacharya2019bayesian} (Theorem 3.2). It is based on a general preliminary step involving fractional Bayesian posteriors; for the sake of simplicity, we have stated our consistency theorem on a particular case of this general result (see Appendix~\ref{app}). Lastly, note that our result of consistency involves the $\|\cdot\|_{L^2}$ norm, which is natural since our algorithm aims to minimize the mean squared error (MSE), whereas in~\cite{linero2018bayesian} the $\|\cdot\|_{\infty}$ norm is considered.




\section{Experiments}\label{sec:exps}
We apply the methods introduced in this paper to several data sets and compare the performances with state-of-the-art methods. 
\subsection{Setting}
\paragraph{Data sets} 
We make use of 16 data sets of various sizes. Most are introduced in \cite{Kim_2007}, namely Hatco with two outputs known as satisfactory (HS) and usage (HU), {\href{https://archive.ics.uci.edu/ml/datasets/Servo}{Servo}} (SE), Ais with two outputs known as sex (ASX) and sport (ASP), \href{http://lib.stat.cmu.edu/datasets/boston}{Boston} (BO), {\href{https://www4.stat.ncsu.edu/~boos/var.select/diabetes.tab.txt}{Diabetes}} (DI), {\href{https://www.rdocumentation.org/packages/missMDA/versions/1.14/topics/ozone}{Ozone}} (OZ), Diamond (DD), {\href{https://archive.ics.uci.edu/ml/datasets/abalone}{Abalone}} (AB), Tecator with three different outputs (T1), (T2), and (T3). We also used
{\href{https://rdrr.io/rforge/hdi/man/riboflavin.html}{Riboflavin}}(RI), which was used in \cite{Buhlmann_2014}, {\href{archive.ics.uci.edu/ml/datasets/bike+sharing+dataset}{Bike-Day}} (BD), and {\href{https://rdrr.io/cran/caret/man/BloodBrain.html}{Blood Brain Barrier data}} (BB).
 
In the experiments reported here, we use the original data sets with standard scaling. Table~\ref{tab:data set-description} below provides the main characteristics of these data sets. 

\begin{table} \centering
\caption{\label{tab:data set-description}Benchmark data sets with their characteristics.}
\begin{tabular}{ccccccccc}
Data set & RI &HS & HU &OZ &SE & ASX & ASP & BB \\
\hline
Features &4088& 12 & 12 &10 &4 & 11 &11 & 133\\
Samples &71& 100& 100 & 112 & 167 & 202&202 & 208\\
\hline
Data set &T1 & T2& T3& DD & DB & AB & BO & BD \\ \hline
Features& 122 &122 & 122& 4 & 10 & 7 & 11 & 11\\
Samples & 215& 215 & 215 & 308 & 442 & 500 & 506 & 731
\end{tabular}
\end{table}

\paragraph{Methods} 
For standard regression trees and their ensemble extensions, we use the implementation from Scikit-Learn \citep{Scikit-learn}, except for BART, which is only available in \href{https://github.com/JakeColtman/bartpy}{github}.
{PR trees} is built on top of this implementation, while a {\href{https://gitlab.com/sami.kh/pr-tree}{git repository}} contains all ensemble extensions of PR trees proposed in this paper. To speed up the algorithm, a subset of splitting variables is used here as proposed in \cite{Alkhoury_2020}, consisting of the top three variables according to the splitting criterion of standard regression trees.
To compute the weights $\boldsymbol{\gamma}$, we rely on the Moore-Penrose pseudo-inverse.
In addition, we use the normal distribution for $\phi$ (Eq.~\ref{Psi}).
For PR trees and their gradient-boosted extension, the validation set is used to estimate the noise vector $\boldsymbol{\sigma}$ using a grid search that takes values for each variable $j, \, 1 \le j \le p$, in the interval $[0, 2{\hat{\sigma}^{(\ell)}}]$ with a step of ${\hat{\sigma}^{(\ell)}}/4$, where ${\hat{\sigma}^{(\ell)}}$ denotes the empirical standard deviation of variable $j$. The noise vector for PR forests is fixed to that obtained for a single tree, as each tree aims to predict the output variable $Y$. 
For soft trees, we use the implementation available in {\href{https://github.com/oir/soft-tree}{github}} with the default parameters and normalized data. We implement the extensions for the RF and GBT. We use the implementation of soft BART available in {\href{https://github.com/theodds/linero2018bayesian}{github}}.
For STR trees and BooST, its extension to GBT, we use the implementation available in {
\href{https://github.com/gabrielrvsc/BooST}{github}}. We implement the extension for the RF. 
For RF-based methods, we use 100 trees and the default parameters (consisting of all variables and a sampling strategy over observations consisting of bootstraps with replacements). For GBT-based methods, we use 50 trees.
Lastly, for both PR and standard regression trees, the stopping criterion is the same in all experiments: all leaves should contain, in a hard-assignment approach, at least 10\% of the training data. The stopping criterion for STR trees is based on the number of regions, which is chosen to be equal to those obtained for PR and standard trees. For soft trees, the growth is stopped when the performance on a validation set decreases, which gives them a slight advantage in terms of performance. 
For the Bayesian extensions, data first have to be normalized. 
\begin{table}
\caption{Results for one tree obtained with a 10-fold stratified cross-validation strategy on PR, standard, soft, and STR trees. '*' corresponds to the best result, and bold highlights the results that are not significantly different from the best one. The noise vector for the PR tree is tuned using cross-validation.\label{tab:1tree}}

\begin{tabular}{*{5}{c}}
Dataset & PR Tree&Standard Tree&	Soft Tree&	STR tree \\
\hline
RI		&{\bf 0.67(0.13)}&	1.04(0.15)	&{\bf 0.66(0.17)*}&	0.85(0.21)\\
HS	&	{\bf 0.44(0.09)}	&0.55(0.09)&	{\bf 0.39(0.06)*}	&0.77(0.06)\\
HU	&	{\bf 4.25(0.31)*} &	5.07(0.6)&	4.51(1.07)&	7.62(1.01)\\
OZ		& {\bf 18.66(3.65)*}	&{\bf 18.9(3.39)	}&34.44(43.27)	&24.5(3.58)\\
SE	&	{\bf 0.84(0.19)}	&{\bf 0.85(0.22)}&	{\bf 0.69(0.19)*}&	1.35(0.17)\\
ASX	&	0.2(0.02)	&0.26(0.06)	&{\bf 0.15(0.03)*}	&0.43(0.01)\\
ASP	&	{\bf 2.45(0.2)*}&	{\bf 2.57(0.17)}&	{\bf 2.56(0.19)}&	{\bf 2.57(0.15)}\\
BB		&{\bf 0.64(0.1)}&	{\bf0.68(0.05)}&	{\bf0.62(0.1)*}	&0.71(0.07)\\
T1	&	{\bf3.57(0.84)*}	&6.06(1.05)	&4.51(0.95)	&7.97(0.66)\\
T2	&	{\bf 4.44(1.08)*}	&7.8(1.69)	&8.39(2.04)&	10.33(1.06)\\
T3	&	1.67(0.45)	&2.21(0.37)	&{\bf1.15(0.25)*}	&2.7(0.26)\\
DD	&	{\bf 647.5(134.41)*}&	1217.69(247.26)	&2620.54(302.97)	&2808.54(199.96)\\
DI&{\bf55.92(3.97)}&60.95(3.92)	&64.18(4.15)&	67.89(1.99)\\
AB&{\bf3.11(0.27)}	&{\bf3.15(0.29)}&	{\bf3.11(0.23)}	&{\bf3.03(0.24)*}\\
BO&{\bf4.47(1.04)}	&5.27(0.61)&	4.54(0.97)&	7.91(0.42)\\
BD&	{\bf898.1(55.4)*}		&1006.7(52)	&1376.7(165.2)&1661.4(54.2)
\end{tabular}
\end{table}
\paragraph{Setting}
To compute the root mean squared error (RMSE), we use a 10-fold stratified cross-validation strategy to evaluate the performance of each method. Each fold is divided into 80\% for training and 20\% for testing, except for soft trees and PR trees and their gradient-boosted extension (see below) for which each fold is divided into 65\% for training, 15\% for validation, and 20\% for testing. For soft trees, the validation set is used for the stopping criterion. 

Lastly, to assess whether the differences are significant or not, we use the following approach. First, a Shapiro-Wilks test is used to determine if the distribution is Gaussian. If true, a t-test (with Bartlett correction for variation homogenization) is used; if false, a Wilcoxon-Mann-Whitney test is used. All the tests are performed at the significance level of 0.05. 

\subsection{Global comparisons}\label{sec:exps:2}
\paragraph{One tree}
In this paragraph, we compare the results of the different methods using one tree, summarized in Table \ref{tab:1tree}. 
Globally, PR tree is the best method. Of the 16 data sets, PR tree is the best method for seven data sets and statistically equal to the best method for seven data sets. Soft tree has also good performances, but surprisingly, STR has poor performances. 

\paragraph{RF and GBT}
In this paragraph, we study the bagging and boosting extensions for the four methods used in the experiments. 
With RF, summarized in Table \ref{tab:RF}, the best method is once again PR-RF for 13 data sets. Results are slightly improved with respect to PR tree. Soft RF and STR-RF did not take advantage of bagging, with most of the performances being similar for one tree and for 100 trees (or even worse). 
A different conclusion may be drawn for GBT, summarized in Table \ref{tab:GBT}. The best methods are PR GBT and BooST (GBT extension of STR trees), which have the best performances for six and eight data sets, respectively, and are always statistically equivalent to the best method (except for the dataset BD for PR GBT). Results are improved here for most of the data sets, although we cannot conclude that bagging is always the best method. 
\begin{table}
\caption{Results for bagging methods with a 10-fold stratified cross-validation strategy for PR, standard, soft, and STR-RF (with 100 trees). '*' corresponds to the best result, and bold highlights the results that are not significantly different from the best result at $5\%$.\label{tab:RF}}

\begin{tabular}{*{5}{c}}
Dataset & PR-RF&Std RF&Soft RF&STR-RF \\ 
\hline
\footnotesize
RI &{\bf 0.64(0.16)*}	&{\bf 0.71(0.19)} &{\bf 0.74(0.19)} & 0.84(0.16 ) \\
HS &{\bf 0.42(0.05)*}&	{\bf 0.46(0.07)}&{\bf 0.44(0.09)} &	0.75(0.08) \\
HU &{\bf 3.69(0.3)*}	&{\bf 4.12(0.43)}	& 4.37(0.92) & 7.75(0.98) \\
OZ &{\bf 16.08(2.37)*}	&{\bf 16.28(2.59)} & 22.79(5.22) & 24.28(3.66)\\
SE &{\bf 0.75(0.21)*}&{\bf 0.75(0.21)} &{\bf 0.77(0.34)} & 1.30(0.22)\\
ASX & 0.19(0.02)&	0.21(0.03) &{\bf 0.13(0.07)*}	& 0.41(0.01) \\
ASP &{\bf 2.39(0.19)*}&{\bf 	2.42(0.2)} & 2.68(0.29) &	2.65(0.78) \\
BB&{\bf 0.56(0.05)*}	&{\bf 0.56(0.06)}	& 0.68(0.10) & 0.71(0.07) \\
T1&{\bf 2.98(0.69)*}	&5.1(1.07) & 6.65(1.57) & 8.39(0.7) \\
T2&{\bf 3.48(1.1)*}&	5.56(1.63) & 9.20(2.88) & 10.85(0.93)\\
T3& 1.51(0.38)	& 1.91(0.32)	& {\bf1.10(0.27)*} & 2.70(0.21) \\
DD&{\bf 635.4(147.9)*}&	972.0(182.0) & 2523.26(396.85) & 2774.36(225.46)\\
DI & {\bf 54.32(3.29)*} &{\bf 56.53(2.61) } & 64.59(4.70) & 70.20(3.64)\\
AB &{\bf 3.07(0.26)} & {\bf 3.09(0.27) } & {\bf 3.05(0.31)*} & {\bf 3.07(0.45)}\\
BO &{\bf 4.01(0.73) *} & {\bf 4.7(0.64)} & 4.81(0.96) & 7.80(0.68)\\
BD & {\bf 838.2(51.1)*} & 907.5(42.4) & 1359.14(133.56) & 1645.74(74.61) \\
\end{tabular}
\end{table}

\begin{table}
\caption{Results for boosting methods with a 10-fold stratified cross-validation strategy for PR, standard, soft GBT, and BooST (with 50 trees). '*' corresponds to the best result, and bold highlights the results that are not significantly different from the best result at $5\%$.\label{tab:GBT}}
\begin{tabular}{*{5}{c}}
Dataset & PR GBT&GBT& Soft GBT &	BooST \\ 
\hline
RI & {\bf 0.68(0.13)} &	0.69(0.12) & {\bf 0.64(0.16)}	&{\bf 0.52(0.09)*} \\
HS & {\bf 0.34(0.29)} & {\bf 0.38(0.05)} & {\bf 0.39(0.06)}	&{\bf 0.33(0.09)*}\\
HU & {\bf 2.45(0.36)*}	& 3.56(0.45) & 3.76(0.85) &	{\bf 2.77(0.40)}\\
OZ & {\bf 15.96(3.13)	}&{\bf 16.11(2.65)}	& 22.61(6.94) & {\bf 15.60(3.31)*} \\
SE & {\bf 0.57(0.19)}	& 	{\bf 0.57(0.19)} & 0.74(0.27) &	{\bf 0.42(0.13)*}\\
ASX & {\bf 0.15(0.03)}	& 0.18(0.03) & {\bf 0.15(0.02)} & {\bf 	0.14(0.02)*}\\
ASP & {\bf 2.42(0.21)}& 2.52(0.23) & {\bf 2.28(0.23)*} &{\bf 2.36(0.26)}\\
BB & {\bf 0.53(0.07)*}&	{\bf 0.56(0.07)} & {\bf 0.55(0.12)}	& {\bf 0.73(0.35)}\\
T1 & {\bf 1.62(0.1)*}	&3.47(0.66) & 3.17(0.84)	&{\bf 1.87(0.22)}\\
T2 & {\bf 1.83(0.29)*}	&4.4(0.96) & 5.96(1.22)	& {\bf 2.39(0.38)} \\
T3 & {\bf 0.86(0.14)*}	&1.37(0.23)& {\bf 0.98(0.21)} &	{\bf 0.95(0.11)}\\
DD & {\bf 514.8(162.8)} &	684.1(178.7)& 1585.99(236.38) &	{\bf 482.45(138.60)*}\\
DI & {\bf 57.14(3.65)} & {\bf 57.33(3.72)} & {\bf 56.73(2.38)} & {\bf 54.38(3.43)*}\\
AB & {\bf 3.15(0.29)} & {\bf 3.15(0.29) } & {\bf 3.01(0.24)*} & {\bf 3.12(0.25)} \\
BO & {\bf 3.4(0.6)*} & {\bf 3.77(0.65)} & {\bf 3.98(0.79)} &{\bf 3.41(0.76) } \\
BD & 683.0(36.6) & 695.8(40.9) & 907.45(93.70) &{\bf 658.32(45.09)*}\\
\end{tabular}
\end{table}

\paragraph{BART}
In this paragraph, we compare the performance for the Bayesian additive regression tree using PR tree, standard tree, or soft tree as blocks, since the extension of the STR tree is not available in the literature. Results are summarized in Table \ref{tab:BART}. 
This ensemble method clearly favors soft BART as the best method for 11 data sets and P-BART for five data sets. 
Although the method is more general and allows for greater variability than RF or GBT, the final results are not better (and are even worse for most of the data sets) compared with the other methods, while the computation time is longer due to the complexity of the modeling. 

\begin{table}\centering
\caption{Results for Bayesian methods with a 10-fold stratified cross-validation strategy for P-BART, BART, and soft BART. '*' corresponds to the best result, and bold represents results that are not significantly different from the best result according to a two-sided t-test at $5\%$.\label{tab:BART}}

\begin{tabular}{llll}
Dataset & P-BART & BART & Soft BART\\
\hline
RI &{\bf 0.62(0.12)} &	{\bf 0.63(0.15)}&	{\bf 0.61(0.17)*} \\
HS &0.38(0.09)	&0.4(0.08)	&{\bf 0.29(0.06)*}	\\
HU&{\bf 2.77(0.39)}	&3.34(0.63)	&{\bf 2.48(0.84)*}\\
OZ& {\bf 15.44(2.73)}	& {\bf 15.46(2.64)}	&{\bf 14.21(2.44)*} \\
SE&0.8(0.14)&	0.68(0.16)&	{\bf 0.43(0.28)*}	\\
ASX& {\bf 0.19(0.02)} &	{\bf 0.19(0.02)}	& {\bf 0.16(0.04)*}	\\
ASP& {\bf 2.4(0.2)*} &	{\bf 2.42(0.22)} &	{\bf 2.45(0.24)} \\
BB& {\bf 0.52(0.06)}	& {\bf 0.55(0.05)}	&{\bf 0.52(0.04)*} \\
T1& {\bf 1.97(0.38)*}	& 4.25(1.33)	& {\bf 2.17(0.37)}	\\
T2&{\bf 2.32(0.63)*}	&5.54(1.97)	&{\bf 2.39(0.51)}	\\
T3&{\bf 0.9(0.2)}&	1.9(0.55)	&{\bf 0.87(0.15)*}	\\
DD&650(120.85)&	720.68(99.43)&	{\bf 305.92(88.81)*}	\\
DI&{\bf 54.06(2.64)*}	&56.26(2.93)&	{\bf 54.46(3.44)}	\\
AB& {\bf 3.06(0.27)*}&	3.13(0.3)&	3.11(0.44)\\
BO& 4.25(0.72)	&4.04(0.69)&	{\bf 3.1(0.42)*} \\
BD& 701.62(48.68)	& 739.69(35.51)&	{\bf 611.41(64.58)*} \\	
\end{tabular}
\end{table}

\subsection{Bias-variance trade-off}
\label{exp:bias_var}
\begin{figure}[t]
 \includegraphics[width=14.5cm]{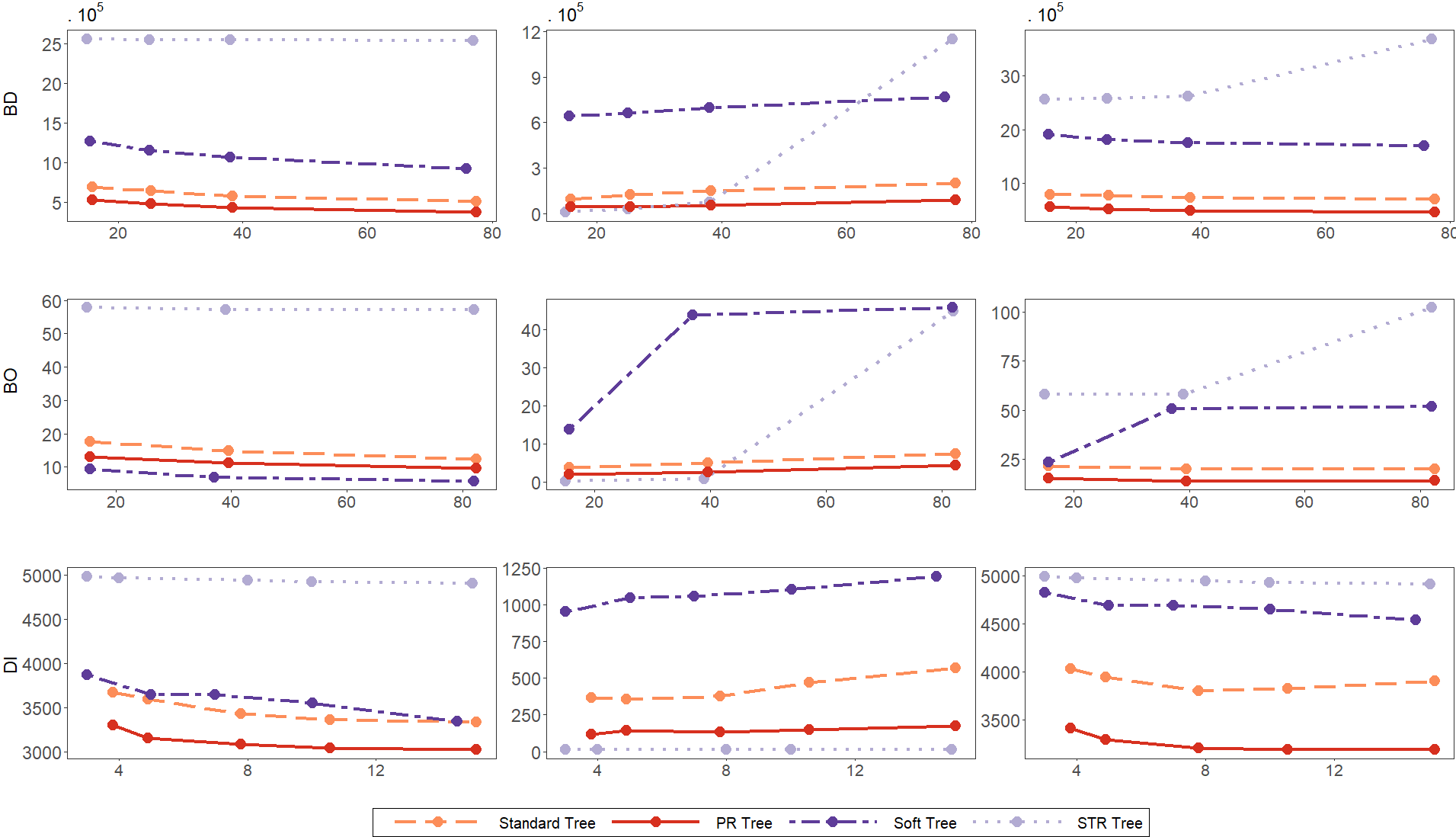}
 \caption{Evolution of the performance with respect to the depth of the tree on several data sets: Bike Day (BD), Boston (BO), Diabetes (DI). Left: bias, middle: variance, right: MSE.\label{fig:bias_var}}
\end{figure}
RMSE, as the standard measure of performance in regression used in Section \ref{sec:exps:2}, summarizes the performance but conceals the details. In this section, we illustrate the evolution of the bias and variance for each estimator introduced in this paper, as it is well known that RMSE can be decomposed as the sum of the squared bias and variance and that each term has a meaning for the ensemble methods. 

The bias and variance are defined in the following way for the prediction function $\hat f$:
\begin{align*}
\text{bias}(\hat{f}) &= \text{E}_{(\B{X},Y)} \left\{ \left( Y - \text{E} \{\hat{f}(\B{X}) \}\right)^2 \right\} ;
\\
\text{var}(\hat{f}) &= \text{E}_{(\B{X},Y)} \left(\text{Var} \{\hat{f}(\B{X}) \}\right) ;
\end{align*}
where the inner expectation and variance are with respect to the estimator. 
To compute the bias and variance, we subsample the data with 80\% for training and 20\% for testing, which estimates the inner and outer expectations. 

The evolution with respect to the dimension is plotted in Fig. \ref{fig:bias_var} for one tree. We vary the number of leaves (changing the stopping criteria) to consider trees of different depths for three data sets (BD, BO, and DI). We remark that the bias decreases with respect to the number of leaves, while the variance increases with respect to the number of leaves, leading to the famous U-shape for RMSE. It is interesting to observe that PR trees improve the bias and variance of the standard trees. STR trees have a particularly large bias, despite the small variance leading to a large RMSE. The opposite is true for soft trees, which can have a good (low) bias but particularly poor (high) variance. 

In Fig. \ref{fig:bias_var_2}, we focus on the Diabetes data set and provide the result for one tree, RF, GBT, and BART. For the ensemble methods (RF, GBT, and BART), the evolution with respect to the dimension is plotted in which we vary the number of aggregated trees.
It is well known that bagging improves the variance (and makes it decrease with the number of trees). This is indeed illustrated in the plots. However, we observe that the standard RF, PR-RF, and STR-RF achieve the same variance. At the end, PR-RF has the best performance in RMSE, because it improves (even for one tree) the bias. Note that the plot of bias is very similar for one tree and RF. It is also known that boosting reduces the bias, as illustrated in this figure. Again, all methods perform similarly, and the gain for PR-RF in RMSE is achieved thanks to the variance reduction for one tree. 
Finally, as BART is a mixture of bagging and boosting (after a warming phase), we acknowledge the improvement in both the bias and variance. All the methods provide comparable results (except for the standard tree that is slightly worse than the others), thus leading to the same RMSE. 

\begin{figure}[t]
 \includegraphics[width=14.5cm]{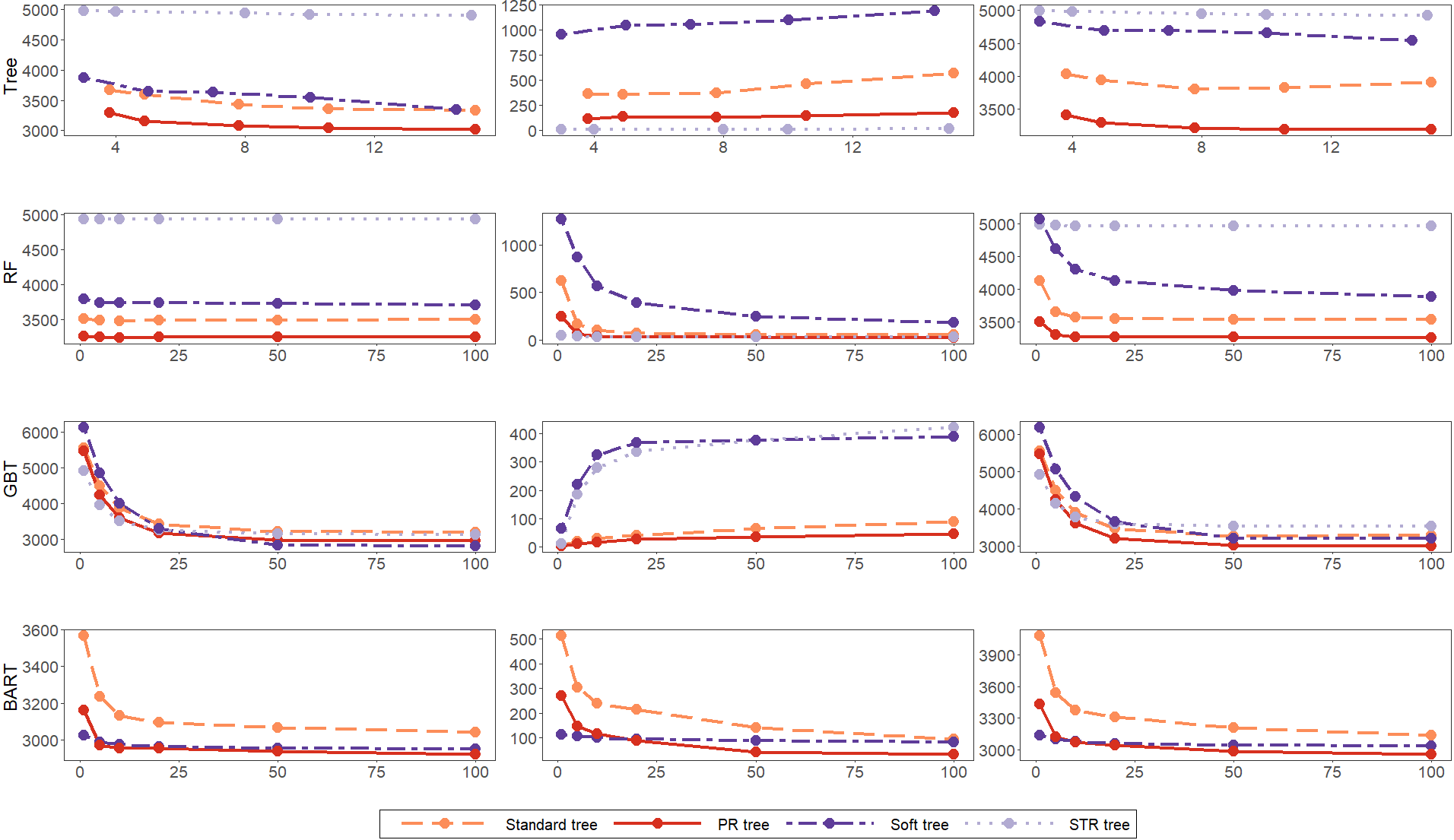}
 \caption{Evolution of the performance of one tree (first row), bagging methods (second row), boosting methods (third row), and Bayesian ensemble methods (fourth row) for the Diabetes data set. Left: bias, middle: variance, right: MSE. 
\\For one tree, we increase the depth of the tree to vary the dimension, whereas for the ensemble methods, we increase the number of trees.\label{fig:bias_var_2}}
\end{figure}

\section{Conclusion}\label{sec:concl} 

Ensemble methods of regression trees are known to be accurate and robust for dealing with general data sets. In this paper, we provide an extension of the standard regression tree, called the probabilistic regression tree, which is more robust to noise and more stable. We study it as a building block for bagging, boosting, and Bayesian ensemble extensions. The consistency of the corresponding parameters is established for all extensions to the ensemble methods. 

Experiments were then carried out to better understand the behavior of each method for several data sets. Unfortunately, there is not one single best method to recommend for any data set, but we rather observe that the ensemble methods converge to similar results. An empirical study of the bias-variance trade-off also illustrates the well-known conclusion and emphasizes which method should be chosen depending on the criterion.

As a future perspective, we would like to generalize the uncertainty considered here for the output using quantile regressions, for example.

\bibliographystyle{chicago}
\bibliography{ref}

\appendix

\section{Technical lemmas}
\subsection{Proof of Lemma~\ref{lem:2a}}
We detail the proof of Lemma~\ref{lem:2a} involved in the proof of Proposition~\ref{prop:step2}. We first need several notations. For each $j$, set $I_j^{(1)}:=[m_j^{(1)}-\delta_j^{(1)},m_j^{(1)}+\delta_j^{(1)}]$
(resp $I_j^{(1)}:=[m_j^{(2)}-\delta_j^{(2)},m_j^{(2)}+\delta_j^{(2)}])$) and $\delta_j^{(1)},\delta_j^{(2)}>0$. Define the two rectangles $R_1:=\prod_j I_j^{(1)}$ and $R_2:=\prod_j I_j^{(2)}$. One has
\begin{lemma}\label{lem:2a}
Assume that $\phi_{\boldsymbol\sigma}\in B^{(1)}_{1,\infty}$. For any $\mathbf x\in [0,1]^p$ and any $\alpha\in (0,1)$, one has:
\[
|\Psi(\mathbf x;R_1)-\Psi( \mathbf x;R_2)|\leq C\max_j\left[\max (|a_j^{(2)}-a_j^{(1)}|,|b_j^{(2)}-b_j^{(1)}|)\right]^{\alpha}
\]
\end{lemma} 
\noindent{\bf Proof.} 
Observe that:
\begin{eqnarray*}
&&|\Psi(\mathbf x;R_1)-\Psi(\mathbf x;R_2)|\\
&=&\left|\int_{R_1}\phi_{\boldsymbol \sigma}(\mathbf u-\mathbf x)d\mathbf u-\int_{R_2}\phi_{\boldsymbol \sigma}(\mathbf u-\mathbf x)d\mathbf u\right|\\
&= &\left|\sum_{j=1}^p \int_{\prod_{\ell=1}^j I_\ell^{(1)}\times \prod_{\ell=j+1}^p I_\ell^{(2)}}\phi_{\boldsymbol \sigma}(\mathbf u-\mathbf x)d\mathbf u-\int_{\prod_{\ell=1}^{j-1}I_\ell^{(1)}\times \prod_{\ell=j}^p I_\ell^{(2)}}\phi_{\boldsymbol \sigma}(\mathbf u-\mathbf x)d\mathbf u\right|\\
&\leq&\sum_{j=1}^p \int_{\prod_{\ell=1}^{j-1}I_\ell^{(1)} \times \prod_{\ell=j+1}^p I_\ell^{(2)}}
\left|\int_{I_j^{(1)}}\phi_{\boldsymbol \sigma}(\mathbf u-\mathbf x)d u_j-\int_{I_j^{(2)}}\phi_{\boldsymbol \sigma}(\mathbf u-\mathbf x)d u_j\right|du_1 \cdots du_{j-1} du_{j+1}\cdots du_p
\end{eqnarray*}

We now bound each term of the sum above. For the sake of simplicity, let us assume that for a given $j$, $\delta_j^{(1)}\leq \delta_j^{(2)}$. Then:
\begin{eqnarray*}
&&\left|\int_{I_j^{(1)}}\phi_{\boldsymbol \sigma}(\mathbf u-\mathbf x)d u_j-\int_{I_j^{(2)}}\phi_{\boldsymbol \sigma}(\mathbf u-\mathbf x)d u_j\right|\\
&=&\left|\int_{m_j^{(1)}-\delta_j^{(1)}}^{m_j^{(1)}+\delta_j^{(1)}}\phi_{\boldsymbol \sigma}(\mathbf u-\mathbf x)d u_j-\int_{m_j^{(2)}-\delta_j^{(2)}}^{m_j^{(2)}+\delta_j^{(2)}}\phi_{\boldsymbol \sigma}(\mathbf u-\mathbf x)d u_j\right|\\
&\leq&\left|\int_{m_j^{(1)}-\delta_j^{(2)}}^{m_j^{(1)}+\delta_j^{(2)}}\phi_{\boldsymbol \sigma}(\mathbf u-\mathbf x)d u_j-\int_{m_j^{(2)}-\delta_j^{(2)}}^{m_j^{(2)}+\delta_j^{(2)}}\phi_{\boldsymbol \sigma}(\mathbf u-\mathbf x)d u_j\right|
\end{eqnarray*}
where the last equality comes from the positivity of $\phi_{\boldsymbol\sigma}$. Performing a change of variable in the first integral, we obtain:
\begin{eqnarray*}
&&\left|\int_{m_j^{(1)}-\delta_j^{(2)}}^{m_j^{(1)}+\delta_j^{(2)}}\phi_{\boldsymbol \sigma}(\mathbf u-\mathbf x)d u_j-\int_{m_j^{(2)}-\delta_j^{(2)}}^{m_j^{(2)}+\delta_j^{(2)}}\phi_{\boldsymbol \sigma}(\mathbf u-\mathbf x)d u_j\right|\\
&\leq&\int_{m_j^{(2)}-\delta_j^{(2)}}^{m_j^{(2)}+\delta_j^{(2)}}\left|\phi_{\boldsymbol \sigma}(\mathbf u-\mathbf x+(m_j^{(1)}-m_j^{(2)}e_j)-\phi_{\boldsymbol \sigma}(\mathbf u-\mathbf x)\right|d u_j\\
&\leq &\int_{\mathbb{R}}\left|\phi_{\boldsymbol \sigma}(\mathbf u-\mathbf x+(m_j^{(1)}-m_j^{(2)}e_j)-\phi_{\boldsymbol \sigma}(\mathbf u-\mathbf x)\right|d u_j
\end{eqnarray*}
Gathering this last inequality, performing the change of variable $\mathbf v:=\mathbf u-\mathbf x$, and using the inequality
\begin{eqnarray*}
&&|\Psi(\mathbf x;R_1)-\Psi(\mathbf x;R_2)|\\
&\leq& \sum_{j=1}^p \int_{\prod_{\ell=1}^{j-1}I_\ell^{(1)} \times \prod_{\ell=j+1}^p I_\ell^{(2)}}
\left|\int_{I_j^{(1)}}\phi_{\boldsymbol \sigma}(\mathbf u-\mathbf x)d u_j-\int_{I_j^{(2)}}\phi_{\boldsymbol \sigma}(\mathbf u-\mathbf x)d u_j\right|du_1 \cdots du_{j-1} du_{j+1}\cdots du_p
\end{eqnarray*}
yields:
\[
|\Psi(\mathbf x;R_1)-\Psi(\mathbf x;R_2)|\leq \sum_{j=1}^p \int_{\mathbb{R}^p}\left|\phi_{\boldsymbol \sigma}(\mathbf v+(m_j^{(1)}-m_j^{(2)})e_j)-\phi_{\boldsymbol \sigma}(\mathbf v)\right|d \mathbf v
\]
We now use the regularity assumption of the function $\phi_{\boldsymbol\sigma}$ which yields that for each $j$ and each $\alpha\in (0,1)$:
\begin{eqnarray*}
\int_{\mathbb{R}^p}\left|\phi_{\boldsymbol \sigma}(\mathbf v+(m_j^{(1)}-m_j^{(2)})e_j)-\phi_{\boldsymbol \sigma}(\mathbf v)\right|d \mathbf v&\leq C|m_1^{(j)}-m_2^{(j)}|^\alpha\leq C\|m_1-m_2\|_\infty^\alpha\\
& C'\,\max(\|a_1-a_2\|_\infty,\|b_1-b_2\|_\infty)
\end{eqnarray*}

\subsection{Technical lemma about the decomposition of a single tree into a sum of trees}
We need preliminary notations to index the regions differently based on the fact that the partition is regular. Any $k\in \{0,\cdots,2^{i_\text{depth}}-1\}$ can be written in an unique way:
\[
k=\sum_{i=0}^{i_\text{depth}-1}\varepsilon_i[k] 2^i
\]
Then, any region $R_k$ can be indexed by $\boldsymbol \varepsilon[k]:=(\varepsilon_i[k])_i$. We can describe the topology $\mathcal{T}$ of any tree of $h$ giving the activated nodes. For example, the topology associated with $h$ will be described as:
\[
\mathcal{T}:=\{(\varepsilon_i)_i\in\{0,1\}^{i_\text{depth}} \}
\]
If $i< i_\text{depth}$, the corresponding region associated with a given node $(\varepsilon_1,\cdots,\varepsilon_i)$ is then:
\begin{equation}\label{eq:union-regions}
\mathcal{R}_{(\varepsilon_1,\cdots,\varepsilon_i)}:=\bigcup_{(\varepsilon_{i+1},\cdots,\varepsilon_{i_\text{depth}})\in \{0,1\}^{i_\text{depth}-i-1}}\mathcal{R}_{(\varepsilon_1,\cdots,\varepsilon_i,\varepsilon_{i+1},\cdots,\varepsilon_{i_\text{depth}})}
\end{equation}

\begin{lemma}\label{lem:one-tree-to-many}
Let $(R_k)$ be a regular partition of depth $i_\text{depth}$ and $h$ a tree of the form:
\[
h:=\sum_{k=0}^{2^{i_\text{depth}}-1} \gamma_k \Psi(\cdot, \mathcal{R}_k)
\]
Then $h$ can we written as a sum of $2^{i_\text{depth}-1}$ trees $(h^{(\ell)})_{\ell=0,\cdots,2^{i_\text{depth}-1}-1}$ with
\[
\|h^{(\ell)}\|_{\infty,\mathcal{V}_n}\leq \max_{i}\max_{(\varepsilon_1,\cdots,\varepsilon_{i})}|\gamma_{\varepsilon_1[\ell],\cdots,\varepsilon_{i-1}[\ell],1-\varepsilon_i[\ell]}-\gamma_{\varepsilon_1[\ell],\cdots,\varepsilon_{i-1}[\ell],\cdot}|
\]
where for any $\ell\in \{0,\cdots,2^{i_\text{depth}-1}\}$, $\boldsymbol \varepsilon[\ell]=(\varepsilon_i[\ell])$ its binary encoding. 
\end{lemma}
\begin{remark}
Note that the way in which $h$ is decomposed is not unique.
\end{remark}
\noindent{\bf Proof:} \\
\noindent{\bf Step 1: definition of the topology of each tree $h^{(\ell)}$}\\
We first begin by defining the topology $\mathcal{T}^{(\ell)}$ associated with each tree $h^{(\ell)}$. Activated nodes of $h^{(\ell)}$ are defined as follows:
\[
\mathcal{T}_{l}:=\{(1-\varepsilon_1[\ell]),(\varepsilon_1[\ell],1-\varepsilon_2[\ell]),\cdots,(\varepsilon_1[\ell],\varepsilon_2[\ell],\cdots,\varepsilon_{i_\text{depth}}[\ell]),(\varepsilon_1[\ell],\varepsilon_2[\ell],\cdots,1-\varepsilon_{i_{depth}}[\ell])\}
\]
We use an example proposed in \cite{rockova2020} for $i_\text{depth=3}$. Let us define the topology of $h$ in Figure~\ref{fig:tree}, and the topology of $h_1,h_2,h_3,h_4$ in Figure~\ref{fig:trees}.

\begin{figure}\centering
		\begin{tikzpicture}[{black, circle, draw, inner sep=1}]
		\tikzset{nodes={draw,rounded corners},minimum height=0.5cm,minimum width=0.5cm, font=\footnotesize}
		\node (X) at (0,1){};
		\node (Y) at (-2,0){};
		\node (Z) at (2,0){};
		\node (A) at (-1,-1){};
		\node (B) at (-3,-1){};
		\node (C) at (1,-1){};
		\node (D) at (3,-1){};
		\node (E) at (-3.5,-2){};
		\node (F) at (-2.5,-2){};
		\node (G) at (-1.5,-2){};
		\node (H) at (-0.5,-2){};
		\node (I) at (0.5,-2){};
		\node (J) at (1.5,-2){};
		\node (K) at (2.5,-2){};
		\node (L) at (3.5,-2){};
		\draw (X) -- (Y);
		\draw (X) -- (Z);
		\draw (Y) -- (A);
		\draw (Y) -- (B);
		\draw (Z) -- (C);
		\draw (Z) -- (D);
		\draw (B) -- (E);
		\draw (B) -- (F);
		\draw (A) -- (G);
		\draw (A) -- (H);
		\draw (C) -- (I);
		\draw (C) -- (J);
		\draw (D) -- (K);
		\draw (D) -- (L);
		\end{tikzpicture}
  \caption{Topology $\mathcal{T}$ of $h$}
  \label{fig:tree}
\end{figure}
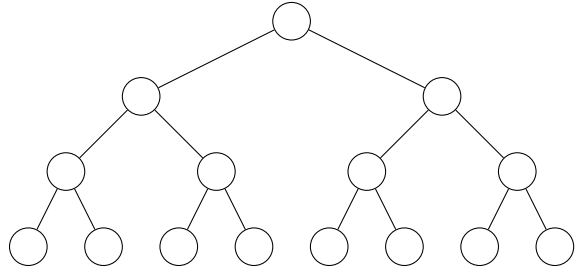

\begin{figure}\centering
		\begin{tikzpicture}[{black, circle, draw}]
		\tikzset{nodes={draw,rounded corners},minimum height=0.3cm,minimum width=0.3cm, font=\footnotesize}
		\node (X) at (-2,1){};
		\node (Y) at (-2.5,0){};
		\node (Z) at (-1.5,0){};
		\node (A) at (-2,-1){};
 		\node (B) at (-3,-1){};
		\node (E) at (-3.5,-2){};
		\node (F) at (-2.5,-2){};
		\draw (X) -- (Y);
		\draw (X) -- (Z);
		\draw (Y) -- (A);
		\draw (Y) -- (B);
		\draw (B) -- (E);
		\draw (B) -- (F);
		\end{tikzpicture}
				\begin{tikzpicture}[{black, circle, draw}]
		\tikzset{nodes={draw,rounded corners},minimum height=0.3cm,minimum width=0.3cm, font=\footnotesize}
		\node (X) at (-1.5,1){};
		\node (Y) at (-2,0){};
		\node (Z) at (-1,0){};
		\node (A) at (-1.5,-1){};
 		\node (B) at (-2.5,-1){};
 		\node (G) at (-2,-2){};
 		\node (H) at (-1,-2){};
		\draw (X) -- (Y);
		\draw (X) -- (Z);
		\draw (Y) -- (A);
		\draw (Y) -- (B);
 		\draw (A) -- (G);
 		\draw (A) -- (H);
		\end{tikzpicture}
		\begin{tikzpicture}[{black, circle, draw}]
		\tikzset{nodes={draw,rounded corners},minimum height=0.3cm,minimum width=0.3cm, font=\footnotesize}
		\node (X) at (-1,1){};
		\node (Y) at (-1.5,0){};
		\node (Z) at (-0.5,0){};
 		\node (C) at (0,-1){};
 		\node (D) at (-1,-1){};
		\node (I) at (-0.5,-2){};
 		\node (J) at (-1.5,-2){};
		\draw (X) -- (Y);
		\draw (X) -- (Z);
 		\draw (Z) -- (C);
 		\draw (Z) -- (D);
		\draw (D) -- (I);
 		\draw (D) -- (J);
		\end{tikzpicture}
		\begin{tikzpicture}[{black, circle, draw}]
		\tikzset{nodes={draw,rounded corners},minimum height=0.3cm,minimum width=0.3cm, font=\footnotesize}
		\node (X) at (0,1){};
		\node (Y) at (-0.5,0){};
		\node (Z) at (0.5,0){};
 		\node (C) at (1,-1){};
 		\node (D) at (0,-1){};
		\node (I) at (0.5,-2){};
 		\node (J) at (1.5,-2){};
		\draw (X) -- (Y);
		\draw (X) -- (Z);
 		\draw (Z) -- (C);
 		\draw (Z) -- (D);
		\draw (C) -- (I);
 		\draw (C) -- (J);
		\end{tikzpicture}
  \caption{Topology $\mathcal{T}_1,\mathcal{T}_2,\mathcal{T}_3,\mathcal{T}_4$ of $h_1,h_2,h_3,h_4$}
  \label{fig:trees}
\end{figure}
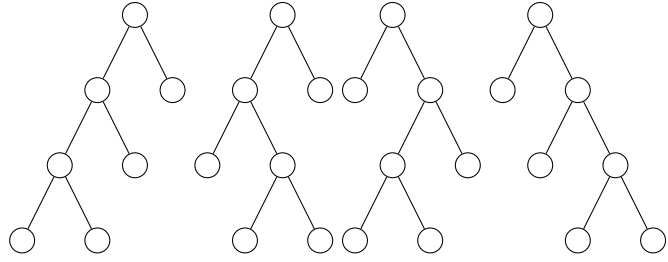

\noindent{\bf Step 2: definition of each tree $h^{(\ell)}$}\\
We provide a first definition of each $h^{(\ell)}$
\begin{eqnarray*}
h^{(\ell)}&=&\gamma_{(\varepsilon_1[\ell],\varepsilon_2[\ell],\cdots,\varepsilon_{i_\text{depth}}[\ell])}\Psi(\cdot,\mathcal{R}_{(\varepsilon_1[\ell],\varepsilon_2[\ell],\cdots,\varepsilon_{i_\text{depth}}[\ell])})\\
&&+\gamma_{(\varepsilon_1[\ell],\varepsilon_2[\ell],\cdots,1-\varepsilon_{i_\text{depth}}[\ell])}\Psi(\cdot,\mathcal{R}_{(\varepsilon_1[\ell],\varepsilon_2[\ell],\cdots,1-\varepsilon_{i_\text{depth}}[\ell])})
\end{eqnarray*}
We have only non-vanishing weights for regions that correspond to the regions of the initial tree $h$. When the weights are non-vanishing, they are identical to those of the initial tree. In this way, we automatically have for any $\ell$:
\[
\|h^{(\ell)}\|_{\infty,\mathcal{V}_n}\leq \|h\|_{\infty,\mathcal{V}_n}
\]

Let $\delta:= \max_j \max (|\widetilde{a}_j-a_j|,|\widetilde{b}_j-b_j|)$. One has $\widetilde{R}\subset R_{\delta}=\{\mathbf y+\delta \mathbf{z}\in\mathbb{R}^p,\, \mathbf y\in R,\mathbf z\in [-1,1]^p\}$. Since in addition
\[
|\Psi(\mathbf x;R)-\Psi(\mathbf x;\widetilde{R})|\leq |\Psi(\mathbf x;R)-\Psi(\mathbf x;R+\mathbf{\delta x})|
\]
one can assume that $\widetilde{R}:=R_{\delta}$. The assumption $\phi_{\boldsymbol\sigma}\in B^{(1)}_{1,\infty}$ implies that for $\alpha\in (0,1)$
\[
\int_{\mathbf x\in \mathbb{R}^p}|\phi_{\boldsymbol\sigma}(\mathbf x+\mathbf{\delta x})-\phi_{\boldsymbol\sigma}(\mathbf x)|d\mathbf x\leq C|\mathbf{\delta x}|^{\alpha}\;.
\]

\section{Proof of Theorem~\ref{th:PR-BART-main}}\label{proof:bart}
\label{app} 
To prove Theorem~\ref{th:PR-BART-main}, we first prove an approximation result of any function $f$ by a sum of trees (Proposition~\ref{prop:step1}). Then we deduce a result for the prior concentration rate (Proposition~\ref{prop:step2}). Thereafter, our result concerning posterior convergence, Theorem~\ref{th:PR-BART-main}, follows from Proposition~\ref{prop:step2} and Theorem 3.2 of~\cite{bhattacharya2019bayesian} as in \cite{linero2018bayesian}. The remainder of this Section is devoted to the proof of the two intermediate results: Proposition~\ref{prop:step1} (see Section~\ref{sec:proof:step1}) and Proposition~\ref{prop:step2} (see Section~\ref{sec:proof:step2}).

\subsection{Proof of Proposition~\ref{prop:step1}}\label{sec:proof:step1}

To prove Proposition~\ref{prop:step1}, we first provide an approximation result in the case of one tree, Proposition~\ref{pro:approx1}. 
\begin{proposition}\label{pro:approx1}
Let $s\in (1,2)$ and $f\in H^s([0,1]^p)$. Then, for some $C(\boldsymbol\sigma,p, \phi)>0$, there exists a partition of $[0,1]^p$ into regions $(\mathcal{R}_k)_{1\leq k \leq K}$ and a tree $h$ corresponding to this partition with a maximum depth $C p(1+1/r)\log(1/\varepsilon)$ and total number of nodes $N$ at most $C\,\varepsilon^{-p/r-p}$ with the form: 

\[
h=\sum_{k=1}^{K} \gamma_{k}\Psi(\cdot;\widetilde{\mathcal{R}}_k)\mbox{ with }\|h\|_{\infty,\mathcal{V}_n}:=\max_k |\gamma_k|\leq \frac{C\varepsilon^{-p/s}}{\inf_{|\boldsymbol\omega|\leq \varepsilon^{-2/s}}|\mathcal{F}\phi_{\boldsymbol{\sigma}}(\boldsymbol\omega)|}
\]
such that for any $\mathbf x\in [0,1]^p$
\begin{align*}
&|f(\mathbf x)-h(\mathbf x)| 
\le C(\boldsymbol\sigma,p, \phi)\|f\|_{H^s(\mathbb{R}^p)}\varepsilon
\end{align*}
\end{proposition} 

\noindent{\bf Proof of Proposition~\ref{pro:approx1}.} Let us fix $\varepsilon>0$. We choose $M$ such that $M^{-r}\leq \varepsilon/3$, that is, $M\geq (\varepsilon/3)^{-1/r}$. We now choose a partition (depending on $M$ and then on $\varepsilon$) such that $\mathrm{diam} (\mathcal{R}_k^{(n)}\cap [-M,M]^p)\leq \varepsilon/3$. This can be obtained considering a regular partition. If $p=1$, this regular partition has depth $i_{depth}$ such that $(2M)/2^{i_{depth}}\leq \varepsilon/3$, that is, we can choose $i_{depth}\geq C_1\log(M/\varepsilon)$. 
If we have $p$ covariables with $p\geq 2$, we repeat this process along each variable, and depth $i_{depth}$ will be at most $C p\log(M/\varepsilon)$. The number of leaves $K$ is of order $2^{i_{depth}}$, that is: 
\[
K\sim C\left[M/\varepsilon\right]^{p}\sim C\,\varepsilon^{-p/r-p}
\]
To finish the proof, we consider $\lambda$ such that:
\[
\frac{1}{(1+|\lambda|)^{s/2}}\leq \varepsilon/3
\]
Thereafter, we choose: 
\[
\beta\geq \frac{C\varepsilon^{-p/s}}{\inf_{|\boldsymbol\omega|\leq \varepsilon^{-2/s}}|\mathcal{F}\phi_{\boldsymbol{\sigma}}(\boldsymbol\omega)|}.
\]
Hence, gathering the inequalities $1/(1+|\lambda|)^{s/2}\leq \varepsilon/3$, $\mathrm{diam} (\mathcal{R}_k^{(n)}\cap [-M,M]^p)\leq \varepsilon/3$, and $M^{-r}\leq \varepsilon/3$ implies the existence of $h\in \mathcal{B}_{\beta}\mathcal{V}_n$ such that for any $\mathbf x\in [0,1]^p$:
\begin{align*}
&|f(\mathbf x)-h(\mathbf x)|\le \left[C(\boldsymbol\sigma,p, \phi)\|f\|_{H^s(\mathbb{R}^p)}\right]\cdot \varepsilon.
\end{align*}
This ends the proof of Proposition~\ref{pro:approx1}.

\noindent Combining Lemma~\ref{lem:one-tree-to-many} and Proposition~\ref{pro:approx1} then yields Proposition~\ref{prop:step1} for $f:=\mathbb{E}[Y|\mathbf{X}=\cdot]$.

\subsection{Proof of Proposition~\ref{prop:step2}}\label{sec:proof:step2}
The proof of Proposition~\ref{prop:step2} follows the same lines as the proof of Theorem~2 in \cite{linero2018bayesian} and is divided into several steps.

We apply Proposition~\ref{prop:step1}, which yields the existence of $\widetilde{h}$, sum of $\widetilde{T}$ trees $h_1,\cdots,h_{\widetilde{T}}$ both belonging to $ \mathcal{B}_{\beta}\mathcal{V}_n$ of respective topologies $\mathcal{T}_1,\cdots,\widetilde{\mathcal{T}}_{\widetilde{T}}$ which approximates $f:=\mathbb{E}[Y|\mathbf{X}=\cdot]$, that is, which satisfies:
\begin{equation}\label{eq:approx-f0}
\|\mathbb{E}[Y|\mathbf{X}=\cdot]-\widetilde{h}\|_{L^2([0,1]^p)}\leq C(\boldsymbol\sigma,p,\phi)\,\|f\|_{H^s(\mathbb{R}^p)}\,\varepsilon/2
\end{equation}
In particular, since $h_1,\cdots,h_T$ belong to $ \mathcal{B}_{\beta}\mathcal{V}_n$ we find that:
\begin{equation}\label{eq:htilde}
\max_{t,j} |\widetilde{\gamma}_{t,j}|=\max_t\|h_t\|_{\infty,\mathcal{V}_n}\leq \beta(\varepsilon)\leq \exp(\varepsilon^{-r_\phi})
\end{equation}
by Equation~\ref{eq:ass-beta}.

We now search to find a lower bound of $\pi\left[\|\widetilde{h}-h\|_{L^2([0,1]^p)}\leq \varepsilon \right]$ where $\pi$ is the prior on all sums of trees $h$. Observe that:
\begin{eqnarray*}
\pi\left[\|\widetilde{h}-h\|_{L^2([0,1]^p)}\leq \varepsilon \right]\geq 
\pi\left[\|\widetilde{h}-h\|_{L^2([0,1]^p)}\leq \varepsilon/2 |\mathcal{T}=\widetilde{\mathcal{T}},T=\widetilde{T}\right]\pi\left[\mathcal{T}=\widetilde{\mathcal{T}}|T=\widetilde{T}\right]\pi\left[T=\widetilde{T}\right]
\end{eqnarray*}
In the next two sections, we give a lower bound of each part on the right hand side of this inequality.
\subsubsection{Lower bound of $\pi\left[\|\widetilde{h}-h\|_{L^2([0,1]^p)}\leq \varepsilon/2 |\mathcal{T}=\widetilde{\mathcal{T}},T=\widetilde{T}\right]$} \label{s:step1}
We denote $(\widetilde{\gamma}_{t,j},\widetilde{R}_{t,j})$ (resp. $(\gamma_{t,j},R_{t,j})$) as the weights and regions associated with $\widetilde{h}$ (resp. $h$). Since the two sums of trees have the same topology, the two sums of trees $h$ and $\widetilde{h}$ are of the form:
\[
h:=\sum_{t=1}^{\widetilde{T}}\sum_{j=1}^{\widetilde{K}_t}\gamma_{t,j}\Psi(\cdot\;;R_{t,j})\mbox{ and }\widetilde{h}:=\sum_{t=1}^{\widetilde{T}}\sum_{j=1}^{\widetilde{K}_t}\widetilde{\gamma}_{t,j}\Psi(\cdot\;;\widetilde{R}_{t,j})
\]
Observe that:
\begin{eqnarray*}
&&\left\|\sum_{t=1}^{\widetilde{T}}\sum_{j=1}^{\widetilde{K}_t}\gamma_{t,j}\Psi(\cdot\;;R_{t,j})-\sum_{t=1}^{\widetilde{T}}\sum_{j=1}^{\widetilde{K}_t}\widetilde{\gamma}_{t,j}\Psi(\cdot\;;\widetilde{R}_{t,j})\right\|_{L^2([0,1]^p)}\\
&\leq& \left\|\sum_{t=1}^{\widetilde{T}}\sum_{j=1}^{\widetilde{K}_t}\gamma_{t,j}\Psi(\cdot\;;R_{t,j})-\sum_{t=1}^{\widetilde{T}}\sum_{j=1}^{\widetilde{K}_t}\widetilde{\gamma}_{t,j}\Psi(\cdot\;;R_{t,j})\right\|_{L^2([0,1]^p)}\\
&&+ \left\|\sum_{t=1}^{\widetilde{T}}\sum_{j=1}^{\widetilde{K}_t}\widetilde{\gamma}_{t,j}\Psi(\cdot\;;R_{t,j})-\sum_{t=1}^{\widetilde{T}}\sum_{j=1}^{\widetilde{K}_t}\widetilde{\gamma}_{t,j}\Psi(\cdot\;;\widetilde{R}_{t,j})\right\|_{L^2([0,1]^p)}
\end{eqnarray*}

First, one has:
\begin{eqnarray*}
&&\left\|\sum_{t=1}^{\widetilde{T}}\sum_{j=1}^{\widetilde{K}_t}\gamma_{t,j}\Psi(\cdot\;;R_{t,j})-\sum_{t=1}^{\widetilde{T}}\sum_{j=1}^{\widetilde{K}_t}\widetilde{\gamma}_{t,j}\Psi(\cdot\;;R_{t,j})\right\|_{L^2([0,1]^p)}\\
&\leq&\left[\max_{t,j}|\gamma_{t,j}-\widetilde{\gamma}_{t,j}|\right]\cdot \left(\int_{[0,1]^p}\left|\sum_{t=1}^{\widetilde{T}}
\sum_{j=1}^{\widetilde{K}_t}|\Psi(\mathbf{x};R_{t,j})|\right|^2d\mathbf{x}\right)^{1/2}\\
&\leq&\left[\max_{t,j}|\gamma_{t,j}-\widetilde{\gamma}_{t,j}|\right]\cdot \widetilde{T}
\end{eqnarray*}
where in the last inequality, for any $t$ and any $\mathbf{x}\in\mathbb{R}^p$, $\sum_{j=1}^{\widetilde{K}_t}|\Psi(\mathbf{x};R_{t,j})|=\sum_{j=1}^{\widetilde{K}_t}\Psi(\mathbf{x};R_{t,j})=1$.

In addition, one has: 
\begin{eqnarray*}
&& \left\|\sum_{t=1}^{\widetilde{T}}\sum_{j=1}^{\widetilde{K}_t}\widetilde{\gamma}_{t,j}\Psi(\mathbf{x};R_{t,j})-\sum_{t=1}^{\widetilde{T}}\sum_{j=1}^{\widetilde{K}_t}\widetilde{\gamma}_{t,j}\Psi(\mathbf{x};\widetilde{R}_{t,j})\right\|_{L^2([0,1]^p)}\\
&\leq&\max_{t,j}|\widetilde{\gamma}_{t,j}| \sum_{t=1}^{\widetilde{T}}\sum_{j=1}^{K_t}\left\|\Psi(\mathbf{x};R_{t,j})-\Psi(\mathbf{x};\widetilde{R}_{t,j})\right\|_{L^2([0,1]^p)}\\
&\leq&\max_{t,j}|\widetilde{\gamma}_{t,j}|\cdot \max\left\|\Psi(\mathbf{x};R_{t,j})-\Psi(\mathbf{x};\widetilde{R}_{t,j})\right\|_{L^\infty([0,1]^p)}\left[\sum_{t=1}^{\widetilde{T}}\sum_{j=1}^{K_t} 1\right]\\
&\leq&C(\alpha,\boldsymbol \sigma,\phi)\cdot \max_{t,j}|\widetilde{\gamma}_{t,j}|\cdot [\max|{x}_{t,j}-\widetilde{{x}}_{t,j}| ]^\alpha\cdot \left[\sum_{t=1}^{\widetilde{T}}{K}_t\right]
\end{eqnarray*}
for any $\alpha\in (0,1)$ and where we consider in the two last inequalities the fact that for any function $f\in L^\infty([0,1]^p)$, $\|f\|_{L^2([0,1]^p)}\leq \|f\|_{L^\infty([0,1]^p)}$, and Lemma~\ref{lem:2a}. 

\noindent We deduce from these two last inequalities the following lower bound:
\begin{eqnarray*}
&&\pi\left[\|\widetilde{h}-h\|_{L^2([0,1]^p)}\leq \frac{\varepsilon}{2}\left|\mathcal{T}=\widetilde{\mathcal{T}},T=\widetilde{T}\right.\right]\\
& \geq& 
\pi\left[\max_{t,j}|\gamma_{t,j}-\widetilde{\gamma}_{t,j}|\leq \frac{\varepsilon}{4\widetilde{T}},\max|x_{t,j}-\widetilde{x}_{t,j}| \leq \left(\frac{\varepsilon}{4\left[\sum_{t=1}^{\widetilde{T}}\widetilde{K}_t\right]\max|\widetilde{\gamma}_{t,j}|}\right)^{1/\alpha}\left|\mathcal{T}=\widetilde{\mathcal{T}},T=\widetilde{T}\right.\right] \\
\end{eqnarray*}
We now use the fact that the trees are drawn independently from each other. Hence:
\begin{eqnarray*}
&&\pi\left[\|\widetilde{h}-h\|_{L^2([0,1]^p)}\leq \frac{\varepsilon}{2}\left|\mathcal{T}=\widetilde{\mathcal{T}},T=\widetilde{T}\right.\right]\\
&\geq& \prod_{t,j}\pi\left[|\gamma_{t,j}-\widetilde{\gamma}_{t,j}|\leq \frac{\varepsilon}{4\widetilde{T}}\left|\mathcal{T}=\widetilde{\mathcal{T}},T=\widetilde{T}\right.\right]\\
&&\cdot \prod_{t,j}\pi\left[|x_{t,j}-\widetilde{x}_{t,j}| \leq \left(\frac{\varepsilon}{4\left[\sum_{t=1}^{\widetilde{T}}\widetilde{K}_t\right]\max|\widetilde{\gamma}_{t,j}|}\right)^{1/\alpha}\left|\mathcal{T}=\widetilde{\mathcal{T}},T=\widetilde{T}\right.\right]
\end{eqnarray*}
By Assumption \eqref{ass:BART}, one has:
\begin{eqnarray*}
&&\prod_{t,j}\pi\left[|\gamma_{t,j}-\widetilde{\gamma}_{t,j}|\leq \frac{\varepsilon}{4\widetilde{T}}\left|\mathcal{T}=\widetilde{\mathcal{T}},T=\widetilde{T}\right.\right]\\
&\geq&C\prod_{t,j}\left[\int_{\widetilde{\gamma}_{t,j}-\varepsilon/(4\widetilde{T})}^{\widetilde{\gamma}_{t,j}+\varepsilon/(4\widetilde{T})}\exp\left(-B|\gamma_{t,j}|\right)d\gamma_{t,j}\right]\\
&\geq&C\prod_{t,j}\left[\exp\left(-B|\widetilde{\gamma}_{t,j}|\right)\cdot \frac{\varepsilon}{4\widetilde{T}}\right]\\
&\geq&C\left[\exp\left(-\sum_{t,j}[B|\widetilde{\gamma}_{t,j}|+\log(4\varepsilon^{-1}\widetilde{T})]\right)\right]\\
&\geq & C\exp\left(-\left[\sum_{t}K_t\right]\cdot\left[B\cdot \max_{t,j}|\widetilde{\gamma}_{t,j}|+\log(4\varepsilon^{-1}\widetilde{T})\right]\right)\\
&\geq & C\exp\left(-\varepsilon^{-p/r-p}\cdot\left[\max_{t,j}|\widetilde{\gamma}_{t,j}|+\log(4\varepsilon^{-1}\widetilde{T})\right]\right)\\
&\geq & C\exp\left(-\varepsilon^{-p/r-p}\cdot\left[\exp(\varepsilon^{-r_\phi})+\log(4\varepsilon^{-1}\widetilde{T})\right]\right)
\end{eqnarray*}
where in the last display we used $\sum_{t=1}^{\widetilde{T}}\widetilde{K}_t \leq \varepsilon^{-p/r-p}$ and the bound on $\max|\widetilde{\gamma}_{t,j}|$ given by Eq.\eqref{eq:htilde}. By Assumption \ref{ass:BART}, one also has for any $\alpha\in (0,1)$: 
\begin{eqnarray*}
&&\prod_{t,j}\pi\left[|x_{t,j}-\widetilde{x}_{t,j}| \leq \left(\frac{\varepsilon}{4\left[\sum_{t=1}^{\widetilde{T}}\widetilde{K}_t\right]\max_{t,j}|\widetilde{\gamma}_{t,j}|}\right)^{\alpha}\left|\mathcal{T}=\widetilde{\mathcal{T}},T=\widetilde{T}\right.\right]\\
&=& \prod_{t,j}\left[\frac{\varepsilon^\alpha}{4\left[\sum_{t=1}^{\widetilde{T}}\widetilde{K}_t\right]^\alpha\max|\widetilde{\gamma}_{t,j}|^\alpha}\right]\\
&= &\exp\left[-\alpha\left[\sum_{t=1}^{\widetilde{T}}\widetilde{K}_t\right]\cdot\log(\varepsilon^{-1}\cdot\max_{t,j}|\widetilde{\gamma}_{t,j}|\cdot\sum_{t=1}^{\widetilde{T}}\widetilde{K}_t)\right]\\
&\geq &\exp\left[-\alpha\varepsilon^{-1/r-p}\cdot\log(\max_{t,j}|\widetilde{\gamma}_{t,j}|\cdot\varepsilon^{-p/r-p-1})\right]\\
&\geq &\exp\left[-\alpha\varepsilon^{-p/r-p}\cdot( \varepsilon^{-r_\phi}+C\log(\varepsilon^{-1})\right]\\
&\geq &C\exp(-\varepsilon^{-p/r-p-r_\phi})
\end{eqnarray*}
where in two last lines we used the inequality $\sum_{t=1}^{\widetilde{T}}\widetilde{K}_t \leq \varepsilon^{-p/r-p}$ and the bound on $\max|\widetilde{\gamma}_{t,j}|$ given by Eq.\eqref{eq:htilde}.
Hence:
\[
\pi\left[\|\widetilde{h}-h\|_{L^2([0,1]^p)}\leq \frac{\varepsilon}{2}\left|\mathcal{T}=\widetilde{\mathcal{T}},T=\widetilde{T}\right.\right]\\
\geq C\exp\left(-\alpha\varepsilon^{-p/r-p}\cdot\left[\exp(\varepsilon^{-r_\phi})+\log(4\varepsilon^{-1}\widetilde{T})+\varepsilon^{-r_\phi}\right]\right)
\]
which implies:
\begin{equation}\label{eq:lb1}
\pi\left[\|\widetilde{h}-h\|_{L^2([0,1]^p)}\leq \frac{\varepsilon}{2}\left|\mathcal{T}=\widetilde{\mathcal{T}},T=\widetilde{T}\right.\right]\\
\geq C\exp\left(-\alpha\varepsilon^{-p/r-p}\cdot\exp(\varepsilon^{-r_\phi})\right)
\end{equation}
\subsubsection{Lower bounds of $\pi\left[\mathcal{T}=\widetilde{\mathcal{T}}|T=\widetilde{T}\right]$ and $\pi\left[T=\widetilde{T}\right]$}

\noindent The topology of the two sums of trees $\mathcal{T}$ and $\widetilde{\mathcal{T}}$ is the same if we draw the same random variables for each node of each tree $t$. Since for each fixed node, the probability of drawing one variable at random is $1/p$, one has: 
\[
\pi\left[\mathcal{T}=\widetilde{\mathcal{T}}|T=\widetilde{T}\right]\geq p^{-\sum_{t=1}^T \widetilde{N}_t}.
\]
In our case:
\[
\sum_{t=1}^{\widetilde{T}} \widetilde{N}_t\leq C\widetilde{T}\log(\varepsilon^{-1}).
\]
Hence:
\begin{equation}\label{eq:lb2}
\pi\left[\mathcal{T}=\widetilde{\mathcal{T}}|T=\widetilde{T}\right]\geq \exp(-C\log(p)\widetilde{T}\log(\varepsilon^{-1}))
\end{equation}

In addition, by Assumption \ref{ass:BART} we have:
\begin{equation}\label{eq:lb3}
\pi\left[T=\widetilde{T}\right]\geq C_1\exp\left(-C_2\widetilde{T}\right).
\end{equation}

Combining~\eqref{eq:lb1},\eqref{eq:lb2}, and \eqref{eq:lb3}, we deduce that: 
\begin{equation}\label{eq:lb}
\pi\left[\|\widetilde{h}-h\|_{L^2([0,1]^p)}\leq \varepsilon/2 \right]\geq C\exp\left(-\varepsilon^{-1/r-p}\cdot\exp(\varepsilon^{-r_\phi})\right)\cdot\exp(-C\log(p)\widetilde{T}\log(\varepsilon^{-1}))\cdot \exp\left(-C_2\widetilde{T}\right)
\end{equation}
\subsubsection{End of the proof of Proposition~\ref{prop:step2}}
\noindent We now provide a bound on the prior $\pi\left[\|h-f\|\leq \varepsilon \right]$ over regression functions $h$. 
Using~\eqref{eq:lb} and \eqref{eq:approx-f0}, one has:
\[
\pi\left[\|h-f\|_{L^2([0,1]^p)}\leq \varepsilon \right]\geq C\exp\left(-\varepsilon^{-p/r-p}\cdot\exp(\varepsilon^{-r_\phi})\right)\cdot\exp(-C\log(p)\widetilde{T}\log(\varepsilon^{-1}))\cdot \exp\left(-C_2\widetilde{T}\right)
\]

Set $\varepsilon_n:=[\log\log n]^{-1/r_\phi}$. Observe that $\varepsilon_n\rightarrow 0$ as $n\rightarrow \infty$ and
\[
\varepsilon_n^{-1/r-p}\cdot\exp(\varepsilon_n^{-r_\phi})=[\log\log n]^{(1/r+p)/r_\phi}\cdot \log n\leq n \cdot [\log\log n]^{-2/r_\phi}
\]
Then,
\[
\pi\left[\|h-f_0\|_{L^2([0,1]^p)}\leq \varepsilon_n \right]\geq C\exp\left(-n\varepsilon_n^2\right)
\]
which yields Proposition~\ref{prop:step2}.

\subsection{Covering number}
Let $N(\varepsilon, \mathcal{F}, \ell_2(Q_n))$ be the empirical $\ell_2$-covering number of $\mathcal{F}$, corresponding to the minimal number of balls with $\ell_2$ distance over observations of radius $\varepsilon$ needed to cover $\mathcal{F}$, and $N_2(\varepsilon, \mathcal{F}, n)$ its uniform version (where the supremum is taken over all uniform probability measures over $n$ points). 
Then, we use the following lemma to bound the cover number.
\begin{lemma}\label{coveringNumber}
Considering the set: 
\begin{align*}
\mathcal{F}_K^{\mathcal{B}} &= \{f \text{ as in \eqref{treeProbabilistic} with } K \text{ leaves} | (\gamma_k)_{1\leq k \leq K} \in [-\Gamma, \Gamma]^K, \boldsymbol{\sigma} \in [a_\sigma, A_\sigma]^p\}.
\end{align*}
The covering number is thus:
\begin{align*}
 \log N_2(\varepsilon, \mathcal{F}_K^{\mathcal{B}}, n) \leq C n \varepsilon^2 K \log(p a_\sigma^{-1}A_\sigma^2 n \varepsilon K\Gamma).
\end{align*}
\end{lemma}

To prove this result, we consider a net corresponding to the center of the balls that recover the considered set of a fixed radius. 

Considering the potentially empty leaves, we can consider a depth $H$ with $K =2^H$ leaves. 

Let: 
\begin{align*}
\mathcal{F}_K^{\mathcal{B}} &= \{f \text{ as in \eqref{treeProbabilistic} with } K \text{ leaves} | (\gamma_k)_{1\leq k \leq K} \in [-\Gamma, \Gamma]^K, \boldsymbol{\sigma} \in [a_\sigma, A_\sigma]^p\}.
\end{align*}

We thus obtain the following bound:
\begin{align*}
 &|f_{\text{PR}} (\mathbf{x};\Theta) - f_{\text{PR}} (\mathbf{x};\widetilde\Theta) | = \left|\sum_{k=1}^K\gamma_k \Psi(\mathbf{x};\mathcal{R}_k,\boldsymbol{\sigma}) - \sum_{k=1}^K\widetilde{\gamma}_k \Psi(\mathbf{x};\widetilde{\mathcal{R}}_k,\widetilde{\boldsymbol{\sigma}})\right| \\
 \leq& \sum_{k=1}^K |(\gamma_k - \widetilde{\gamma}_k) \Psi(\mathbf{x};\mathcal{R}_k,\boldsymbol{\sigma})| + 
 \sum_{k=1}^K|\widetilde{\gamma}_k (\Psi(\mathbf{x};{\mathcal{R}}_k,{\boldsymbol{\sigma}}) - \Psi(\mathbf{x};\widetilde{\mathcal{R}}_k,\widetilde{\boldsymbol{\sigma}}))|
 \\
\end{align*}
We construct an $A_\sigma^{-1} (Cn\varepsilon)^{-1}K^{-1}2^{-K}$-net of $[0,1]$ for the splitting locations, $A_\sigma^{-1} (Cn\varepsilon)^{-1}K^{-1}2^{-K}$-net of $[0, a_\sigma^{-1}]$, and an $(Cn\varepsilon)^{-1}2^{-K}$-net of $[-\Gamma, \Gamma]$ for the leaf values. 
Following the proof of Lemma~1 in the Supplementary Material of~\cite{linero2018bayesian}, which provides the covering number for a similar problem: 
\begin{align*}
 &\|f_{\text{PR}} (\mathbf{x};\Theta) - f_{\text{PR}} (\mathbf{x};\widetilde\Theta) \| 
 \leq K (Cn\varepsilon)^{-1}2^{-K} \|\Psi(\mathbf{x};{\mathcal{R}}_k,{\boldsymbol{\sigma}})\| + (Cn\varepsilon)^{-1}K^{-1}2^{-K}.
 \end{align*}
We can construct a net of $\mathcal{F}_K^{\mathcal{B}}$ with all functions as discretized by each specific net. The total number of such functions can be bounded above by a multiple of:
\[
(A_\sigma (Cn\varepsilon)K2^{K} a_\sigma^{-1} A_\sigma (Cn\varepsilon)K2^{K})^{2^{H+1}}(2\Gamma Cn\varepsilon2^{K})^{2^H},
\]
which gives the upper bound of the covering entropy. 
 
\end{document}